\documentclass[journal]{IEEEtran}
\usepackage{cite}
\usepackage{url}
\usepackage{color}
\usepackage{subcaption}
\usepackage[pdftex]{graphicx}
\usepackage{amssymb,amsmath}
\usepackage{booktabs}
\usepackage{amsmath}
\usepackage{multirow}
\usepackage{bm}
\usepackage[colorlinks,linkcolor=red]{hyperref}
 \usepackage[ruled]{algorithm2e}                 
\usepackage{graphicx}
\usepackage{array}
\usepackage{fixltx2e}
\newtheorem{lemma}{Lemma}[section]

\newtheorem{theorem}{Theorem}[section]
\newtheorem{proof}{Proof}[section]
\newtheorem{definition}{Definition}[section]
\newtheorem{remark}{Remark}[section]

\graphicspath{{Images/}}

\ifCLASSINFOpdf
\else
\fi
\begin{document}
%
\title{Monocular Obstacle Avoidance Based on Inverse PPO for Fixed-wing UAVs}
%
%
%

\author{Haochen Chai\IEEEauthorrefmark{1}, Meimei Su\IEEEauthorrefmark{1}, Yang Lyu\IEEEauthorrefmark{2}, Zhunga Liu, Chunhui Zhao, Quan Pan
\thanks{This work was supported by the National Natural Science Foundation of China under Grant 62203358, Grant 62233014, and Grant 62073264.}
\thanks{\IEEEauthorrefmark{1}Equal contribution.\IEEEauthorrefmark{2}Corresponding author. Haochene Chai, Meimei Su, Yang Lyu, Zhunga Liu, Chunhui Zhao, and Quan Pan are all with the School of Automation, Northwestern Polytechnical University. Xi’an, Shaanxi, 710129, PR China. Email: \url{lyu.yang@nwpu.edu.cn}
}
\thanks{}}

%
%

\markboth{Journal of \LaTeX\ Class Files,~Vol.~14, No.~8, August~2015}%
{Shell \MakeLowercase{\textit{et al.}}: Bare Demo of IEEEtran.cls for IEEE Journals}
%



\maketitle

\begin{abstract}
Fixed-wing Unmanned Aerial Vehicles (UAVs) are one of the most commonly used platforms for the burgeoning Low-altitude Economy (LAE) and Urban Air Mobility (UAM), due to their long endurance and high-speed capabilities. Classical obstacle avoidance systems, which rely on prior maps or sophisticated sensors, face limitations in unknown low-altitude environments and small UAV platforms. In response, this paper proposes a lightweight deep reinforcement learning (DRL) based UAV collision avoidance system that enables a fixed-wing UAV to avoid unknown obstacles at cruise speed over 30m/s, with only onboard visual sensors. The proposed system employs a single-frame image depth inference module with a streamlined network architecture to ensure real-time obstacle detection, optimized for edge computing devices. After that, a reinforcement learning controller with a novel reward function is designed to balance the target approach and flight trajectory smoothness, satisfying the specific dynamic constraints and stability requirements of a fixed-wing UAV platform. An adaptive entropy adjustment mechanism is introduced to mitigate the exploration-exploitation trade-off inherent in DRL, improving training convergence and obstacle avoidance success rates. Extensive software-in-the-loop and hardware-in-the-loop experiments demonstrate that the proposed framework outperforms other methods in obstacle avoidance efficiency and flight trajectory smoothness and confirm the feasibility of implementing the algorithm on edge devices. The source code is publicly available at \url{https://github.com/ch9397/FixedWing-MonoPPO}.
\end{abstract}


%
\IEEEpeerreviewmaketitle

\section{Introduction}
Fixed-wing Unmanned Aerial Vehicles (UAVs) have emerged as key platforms with the development of the Low-Altitude Economy (LAE) \cite{huang2024low} and the rise of Urban Air Mobility (UAM)\cite{cohen2021urban}. Due to the extended endurance and high-speed characteristics of fixed-wing UAVs compared to rotary UAVs\cite{lyu2024fixed}, they are especially preferred in applications such as long-distance cargo delivery \cite{zhang2022adaptive}, wide-area inspection \cite{lungu2020backstepping}, and emergency operations\cite{r1}. When a UAV is flying at low altitudes, autonomously avoiding potential obstacles, such as manmade facilities and terrains,  becomes a key capability to guarantee its flight safety.
Achieving low-altitude obstacle avoidance on a fixed-wing UAV is especially challenging due to its higher speed and inescapably large turning radius compared to other aerial platforms. It requires not only extended environment perception capabilities to realize early collision warning but also more flexible avoidance path planning subject to more rigorous dynamic constraints, to generate feasible avoidance trajectories. 
Classic path planning methods predominantly rely on sampling\cite{karaman2011anytime}, or optimization techniques \cite{mercy2017spline}, which necessitate comprehensive environmental perception, or rely on high-precision prior maps.
The high-speed flight property of fixed-wing UAVs often makes it impractical to produce comprehensive maps or collect accurate environmental data of unknown scenarios\cite{wu2022learning}. Specifically, some scholars utilize costly sensors, such as LiDAR\cite{munoz2022openstreetmap}, or RADAR \cite{popov2023nvradarnet} to enhance environmental perception. However, these approaches not only increase system complexity and cost but also limit their applicability in scenarios with limited onboard resources. Thus, a lightweight navigation framework that removes the dependency on complete environment information or expensive sensors is imperative.
\begin{figure}
    \centering
    \includegraphics[width=1.0\linewidth]{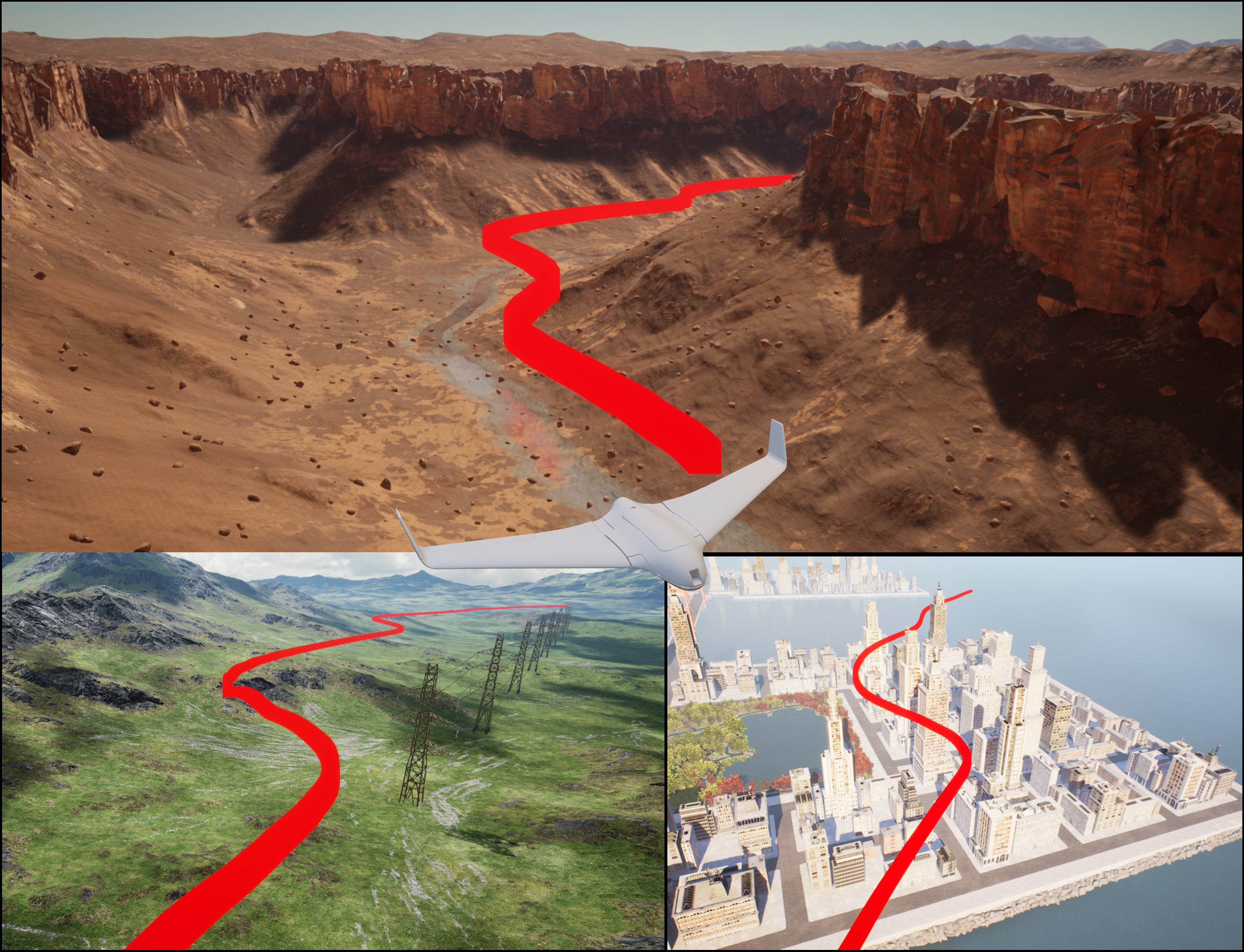}
    \caption{Simulation scenarios and fixed-wing UAV model used for training and validating. Full video link: \url{https://youtu.be/DXP54UI2lbE}}
    \label{fig:fengmian}
\end{figure}
The visual sensor is a most widely used sensor in robot applications with its unparalleled low Size, Weight, and Power consumption (SWaP) and high sensory resolution. 
Compared to classic intelligent heuristic algorithms \cite{mandloi2021unmanned, ma2022bi, wang2022gmr}, DRL demonstrates unique strengths in tackling the obstacle avoidance problem with instantaneous visual information. 
Especially, the success of DRL in gaming applications \cite{kaufmann2023champion} has spurred interest in exploring its potential for visual avoidance \cite{wu2021learn}, a domain where DRL further exemplifies its strength as an end-to-end learning framework \cite{xue2022uav, kulhanek2021visual}. 
However, strategic problems still need to be resolved for integration DRL in fixed-wing UAV obstacle problem. Studies \cite{wu2023human, huang2022vision, wang2019autonomous} have employed reward functions based on distance metrics to produce favorable results. However, these reward functions only consider target arrival and fail to account for trajectory smoothness, which is particularly crucial for flight platforms such as fixed-wing UAVs. When considering the coupling optimization of visual information and obstacle avoidance, some studies use multi-frame depth maps\cite{xie2017towards, de2022depth} as the input of deep reinforcement learning collision avoidance, which undoubtedly increases the systematic latency and jeopardizes the real-time performance, especially on edge computing platforms, which is a critical index for the high-speed fixed-Wing UAV.
Besides the specific strategic problem, a persistent common challenge in utilizing DRL algorithms in robot applications is the imbalance between exploration and exploitation during training. Excessive focus on exploration can hinder algorithm convergence, while insufficient exploration may overlook superior solutions\cite{yuan2022renyi}.

Keen on the above challenges in using DRL to achieve fixed-wing obstacle avoidance, we propose a lightweight deep reinforcement learning framework that utilizes single-frame images captured by low-cost vision sensors as input and exploits the advantages of Proximal Policy Optimization (PPO) \cite{schulman2017proximal} to effectively address obstacle avoidance tasks, offering the following contributions:
\begin{itemize}
\item Considering the flight stability and dynamic constraints of fixed-wing UAVs, we propose an optimization framework formulated as an inverse PPO learning model, which incorporates a reward function balancing target approach and trajectory maintenance to ensure smooth and efficient collision avoidance flight trajectories.
\item We introduce a strategy updating mechanism based on adaptive entropy adjustment to address the challenge of local optimization caused by PPO's reliance on historical data during training. This mechanism ensures that our algorithm identifies obstacle-avoidance strategies with higher success rates.
\item We demonstrate that the proposed framework outperforms other methods in obstacle avoidance efficiency and flight trajectory smoothness through software-in-the-loop and hardware-in-the-loop experiments, and confirmed the feasibility of running the algorithm on edge devices. 
\end{itemize}

The remainder of this paper is organized as follows. Section II reviews related works, and Section III presents the problem definition and its mathematical formulation. Section IV introduces the inverse PPO. Section V discusses the computational experiments and their results. Section VI concludes the paper.

\section{Related Work}

\subsection{Fixed-wing UAV Collision Avoidance}
In contrast to obstacle avoidance algorithms for quadcopters, those for fixed-wing UAVs must account for complex dynamic constraints. For example, a fixed-wing UAV usually has narrower cruise velocity bounds, making it unable to change its velocity abruptly or hover in place like a quadrotor UAV. 

Classic obstacle avoidance algorithms, such as Dijkstra \cite{1959A} and A-star \cite{hart1968formal}, are commonly used in static obstacle environments. However, these methods encounter significant challenges with local minima and often generate trajectories lacking smoothness, particularly in environments with closely spaced obstacles or narrow passages. Another class of algorithms, such as artificial potential field methods \cite{warren1989global}, RRT \cite{noreen2016optimal}, and VFH \cite{babinec2014vfh}, is more suitable for dynamic obstacle environments. Unfortunately, none of these methods address the dynamic constraints of fixed-wing UAVs, necessitating extensive post-processing to smooth the generated paths. To mitigate the reliance on post-processing, many researchers have adopted Dubins curves \cite{mclain2014implementing} for fixed-wing UAV path planning. Dubins curves employ a combination of straight-line and circular arc segments to generate paths that precisely satisfy the kinematic constraints of fixed-wing UAVs.

Different from the above, which puts the obstacle avoidance idea in the planning layer, there are also a large number of studies that consider the obstacle avoidance module in the control layer, for example, approaches based on Model Predictive Control (MPC) \cite{lindqvist2020nonlinear} primarily use optimization theory or continuously update waypoints or routes to prevent collision. Nevertheless, this approach necessitates the creation of highly detailed models of aircraft, with the modeling process for aircraft with varying dynamics being inherently distinct. Consequently, the potential for generalization is limited.

The aforementioned methods face significant challenges when sensors provide only partial or incomplete environmental and obstacle information.
To address this problem, numerous studies have focused on leveraging learning-based algorithms to solve obstacle avoidance under partially observed or unknown environmental conditions.

\subsection{DRL for Visual Navigation}
Deep reinforcement learning, a prominent subfield of machine learning, provides a unique advantage in facilitating interactive and adaptive learning within complex and uncertain environments. This capability makes it a preferred approach among researchers aiming to tackle intricate problems that traditional algorithms struggle to address effectively. From the initial deployment of table storage to address discrete state and action spaces, such as SARSA \cite{zhao2016deep} and Q-learning \cite{watkins1992q}, these techniques can effectively address problems with reduced complexity. Subsequently, the concept of value approximation in neural networks led to the development of Deep Q-learning (DQN) \cite{hernandez2019understanding}, Double DQN \cite{van2016deep}, and Dueling DQN \cite{wang2016dueling}. These algorithms overcome the limitations of previous approaches, effectively enabling the handling of continuous state spaces. Dueling DQN, on the other hand, incorporates an advantageous function to assess the quality of an action within the dual network structure. In contrast to the aforementioned algorithms, which are based on iterative updating of Bellman value functions, DDPG \cite{lillicrap2015continuous}, SAC \cite{haarnoja2018soft}, TRPO \cite{lapan2018deep}, and PPO are based on the theory of gradient descent. Among these options, PPO stands out for its stability, usability, and efficiency. The introduction of a clipping loss function serves to enhance the stability of the training process, limiting the magnitude of each policy update and thereby avoiding the potential for drastic policy changes. Compared to TRPO, PPO simplifies the training process by avoiding complex constrained optimization calculations while keeping policy updates within a safe range. In addition to excelling in benchmark tasks, PPO proves effective in addressing complex real-world problems. However, one limitation of PPO is its high data requirement for training and its heavy reliance on historical data, which may lead to excessive dependence on prior experience if early-stage learning data is insufficient.

The combination of DRL and visual information aims to optimize the navigation efficiency and performance of RL by improving the feature extraction and fusion of information from the perception side. Some researchers have utilized RGB images, depth images, and area-segmented images to guide robots in optimizing navigation and obstacle avoidance strategies. Many works typically rely on comprehensive maps or collect accurate local maps, however, the high-speed flying nature of fixed-wing UAVs often makes it difficult to obtain comprehensive maps or gather accurate data in unknown environments where environmental information is impractical to acquire\cite{wu2021learn,martini2022position,lu2021mgrl, chai2022design}.
Other scholars have enhanced the generalization capabilities of DRL by incorporating human knowledge as prior information, enabling navigation in novel environments with sparse rewards\cite{jiang2021temporal, wu2023human}.
While a substantial quantity of data can be augmented with human knowledge through supervised learning, the acquisition of strategic learning is frequently constrained by the methodologies employed for label generation. Conversely, the acquisition of human knowledge necessitates a considerable investment of effort.

shrink
{In summary, integrating deep reinforcement learning with visual information presents a powerful approach to solving robot navigation and obstacle avoidance challenges. However, existing DRL-based navigation and obstacle avoidance algorithms often suffer from an imbalance between learning and utilization, resulting in prolonged convergence times, low efficiency, and suboptimal obstacle avoidance performance. To address this, our approach incorporates a self-regulating entropy mechanism to enhance reinforcement learning performance. Combined with a backpropagation reward mechanism, this approach significantly improves navigation efficiency in unknown obstacle environments for fixed-wing UAVs.
}

\section{Methodology}
\subsection{Problem Formulation}
The monocular vision-based obstacle avoidance problem can be modeled as a Markov Decision Process (MDP), which is characterized by a tuple $\{\mathcal{S},\mathcal{A},\mathcal{P},r\}$. At time step $t$, the fixed-wing UAV collects environmental state variables $s_t$ using its camera. Based on the state ${s_t \in \mathcal{S}}$, the UAV selects an action \( a_t \) from the action space ${\mathcal{A}}$. The action ${a_t \in \mathcal{A}}$ interacts with the environment, generating a reward signal \( r_t \) and resulting in a transition to the next state \( s_{t+1} \). The objective of the algorithm is to find a policy that maximizes the cumulative reward $\sum_{t=0}^{\infty}\gamma^{t}\cdot r_{t}$ by selecting actions ${a_t}$ that yield the highest expected return at any given time step \( t \).

\subsubsection{State Space}
The state space encompasses the environmental data collected by the camera and the information regarding the target. This can be represented as
\begin{equation}
\mathcal{S} = \left\{\mathcal{S}^{\mathrm{env}}, \mathcal{S}^{\mathrm{tar}}\right\}
\end{equation}
where ${\mathcal{S}^{\mathrm{env}}}$ represents the environment captured by the camera, while ${\mathcal{S}^{\mathrm{tar}}}$ denotes the features related to the target.
${\mathcal{S}^{\mathrm{env}}}$ refers to a latent representation obtained from the encoder of an autoencoder network, designed to reduce redundant and adversarial information. In this context, the RGB image captured by the front-view monocamera $\mathbf{i}_{\mathrm{RGB}}$ is processed to extract depth information, as illustrated below
\begin{equation}
\mathcal{D}={\Gamma}_{\mathrm{depth}}(\mathbf{I_{\mathrm{RGB}}},\theta_{\mathrm{depth}}),
\end{equation}
where $\mathcal{D}\in\mathbb{R}^{H\times W}$ denotes a depth map with dimensions ${H}$ (height) and ${W}$ (width), ${\Gamma}_{\mathrm{depth}}$ is the depth estimation model with parameter $\theta_{\mathrm{depth}}$.

The latent representation is subsequently derived through convolutional encoding of the current generated depth map. This process, at a given time step \( t \), can be expressed as follows
\begin{equation}
\mathbf{f}_{t}={\Gamma}_{\mathrm{enc}}({\mathcal{D}_{t}},\theta_{\mathrm{e}}),
\end{equation}
where $\mathbf{f}_{t}\subseteq\mathbf{f}\in\mathbb{R}^{K}$ denotes the latent variable of size \( K \), while $\Gamma_{\mathrm{enc}}$ represents the encoding function parameterized by $\theta_{\mathrm{e}}$. Accordingly, ${\mathcal{S}^{\mathrm{env}}}$ is derived as follows
\begin{equation}
\mathcal{S}^{\mathrm{env}} = [\mathbf{f}],
\end{equation}
${\mathcal{S}^{\mathrm{target}}}$ represents a local goal, which can be expressed as follows
\begin{equation}
\mathcal{S}^{\mathrm{target}}=[d,\alpha],
\end{equation}
where \( d \) and $\alpha$ represent the normalized relative distance and angle to the goal position, respectively. In this context, we consider a 2-dimensional coordinate system, where \( d \) is computed as follows
\begin{equation}
d=\frac{\|\bm{p}_\mathrm{target},\bm{p}_\mathrm{ego}\|_2}{d_\mathrm{max}},
\end{equation}
where ${{\bm{p}}_{{\rm{ego}}}}$ and ${{\bm{p}}_{{\rm{target}}}}$ represent the coordinate vectors of the current drone position and the target position, respectively. ${{\left\| \cdot \right\|}_{2}}$ denotes the L2 norm, $d_{\mathrm{max}}$ represents the maximum allowable distance between the UAV and the target. $\alpha$ is in radians and is calculated by
\begin{equation}
\alpha=\arctan\left(\frac{\bm{p}_{\text{target},y}-\bm{p}_{\text{ego},y}}{\bm{p}_{\text{target},x}-\bm{p}_{\text{ego},x}}\right)/\pi, 
\end{equation}
where $x$ and $y$ correspond to the longitudinal and lateral axes of the coordinate system, respectively.

\subsubsection{Action Space}
To adapt to the flight characteristics of fixed-wing UAVs, the action space is composed of waypoints in various directions within the body-fixed coordinate system under a constant altitude system, as well as the continuation of the action from the previous time step. This can be formulated as
\begin{equation}
\mathcal{A} = \left\{
\begin{array}{ll}
{{\bm{w}}_{t - 1}} & \text{if continue last action} \\
{{\bm{w}}_{t}} & \text{otherwise}
\end{array},
\right.
\end{equation}
where ${\bm{w}}_{t}$ represents the choosing waypoint and can be calculated as
\begin{equation}
(x_{t}^{body},y_{t}^{body}) \buildrel \Delta \over = {\lambda} \cdot (\cos ({\Delta _{{\rm{yaw}}}}),\sin ({\Delta _{{\rm{yaw}}}})),
\end{equation}
where $\Delta _{\rm{yaw}}\in\left\{ 0, \pm \frac{\pi }{6}, \pm \frac{\pi }{4}, \pm \frac{\pi }{3}\right\}$ represents the discrete desired change in yaw angle magnitude and $\lambda$ represents the Euclidean distance between the calculated waypoint and the current position. 
\subsubsection{Reward Function}
The design of the reward function remains one of the most significant challenges in DRL algorithms. A primary limitation of RL is that reward functions are typically hand-crafted and tailored to specific domains. There has been quite a bit of research in Inverse Reinforcement Learning (IRL), and most of the work provides a way to automatically obtain cost functions from expert demonstrations. However, these approaches are often computationally intensive, and the optimization required to identify a reward function that accurately represents expert trajectories is inherently complex. 
This paper focuses on designing a denser reward function to enhance the obstacle avoidance strategy, aiming not only to achieve high success rates in avoiding obstacles but also to enable smoother paths.
In the process of obstacle avoidance, this paper introduces a reward function that incorporates an inference mechanism to ensure robust learning under conditions of general applicability and rapid convergence.

When the drone reaches its designated target, it immediately receives a reward $r_{\mathrm{target}}$ defined as
\begin{equation}
r_{\mathrm{target}}=\left\{\begin{matrix}
C_1 & \text{if reaches target} \\
0 & \text{otherwise}
\end{matrix}\right..
\end{equation}

\begin{figure*}[htbp]
    \centering
    \includegraphics[width=\linewidth]{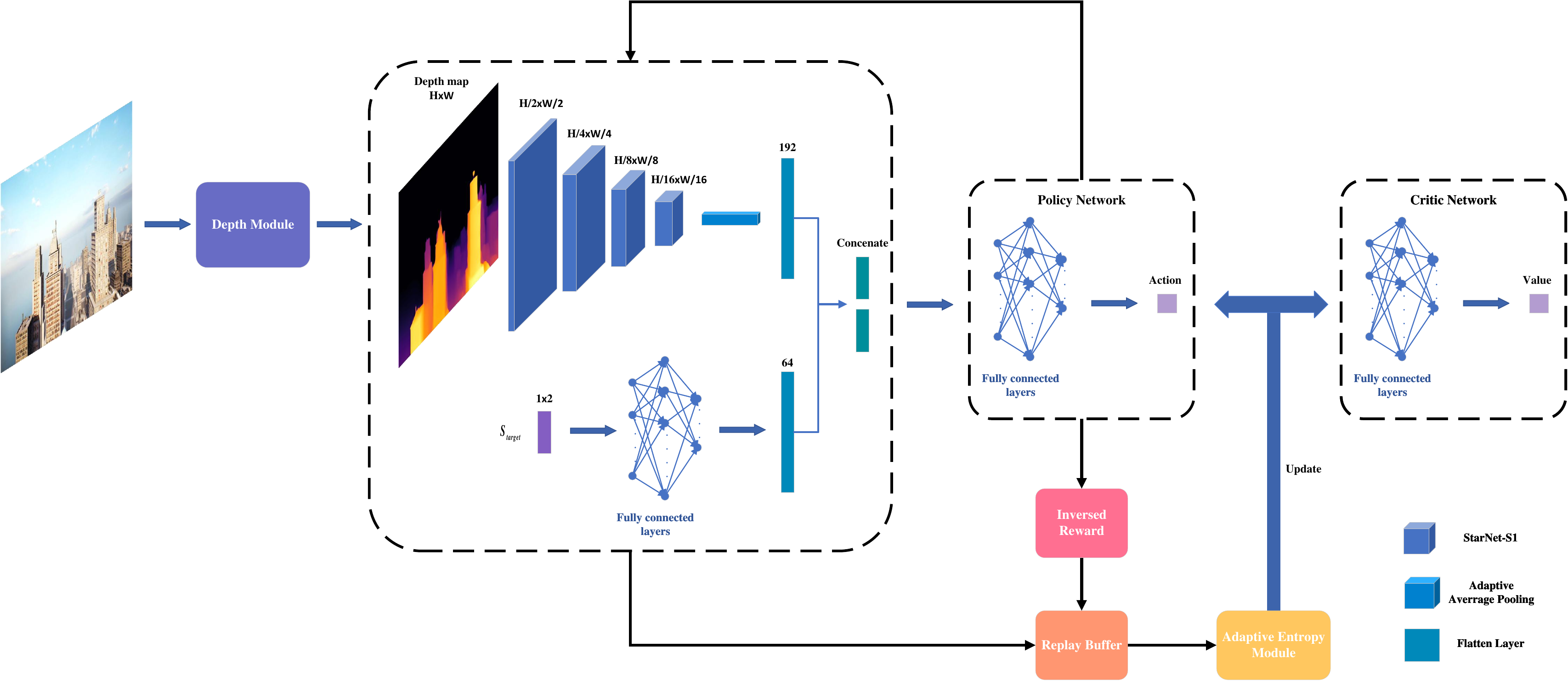}
    \caption{The proposed obstacle avoidance framework for fixed-wing UAVs. A depth map is generated from a monocular RGB image using the method described in \cite{bhat2023zoedepth}, which is encoded by a lightweight backbone \cite{ma2024rewrite} to extract visual features. These visual features are concatenated with target features and input into the policy network to generate actions, while the critic network evaluates state values. An adaptive entropy module dynamically adjusts the exploration-exploitation tradeoff during training, and an inverse reward function updates the replay buffer, facilitating continuous policy optimization.}
    \label{fig:framework}
\end{figure*}
If the drone experiences a collision, it incurs a negative reward $r_{\mathrm{collision}}$ as a penalty defined as
\begin{equation}
r_{\mathrm{collision}}=\left\{\begin{matrix}
C_2 & \text{if collision happens} \\
0 & \text{otherwise}
\end{matrix}\right..
\end{equation}

The drone should approach the target as quickly as possible, making it necessary to encourage the drone to be closer to the target at time  \( t \) than at time \( t-1 \). The corresponding reward $r_{\rm{dis}}$ is defined as
\begin{equation}
{r_{\rm{dis}}} = \Delta d \cdot {C_3},
\end{equation}
\begin{equation}
\Delta d = {d_{t - 1}} - {d_t},
\end{equation}
where ${d_t}$ represents the relative distance between the drone and the target point at time ${t}$.

Finally, we aim for the drone to reach its destination via the shortest path. Therefore, we designed a reward function $r_{\rm{track}}$ that encourages the drone to follow the planned trajectory while learning to interpret depth information to avoid obstacles. The corresponding reward $r_{\rm{track}}$ is defined as
\begin{equation}
{r_{\rm{track}}} = \delta  \cdot {C_4},
\end{equation}
\begin{equation}
\delta  = \frac{{{{\bm{p}}_{{\rm{target}}}} \cdot {{\bm{p}}_{{\rm{ego}}}}}}{{{{\left\| {{{\bm{p}}_{\rm{target}}}} \right\|}_{2}} \cdot {{\left\| {{{\bm{p}}_{\rm{ego}}}} \right\|}_{2}}}},
\end{equation}
where $\forall i \in {\left\{1, \dots, 4\right\}}, {C_i}$ represents the weight of each reward module.

The overall reward function is constructed by combining the four aforementioned sub-terms as follows,
\begin{equation}
\label{eq:reward_function}
{r} = r_{\mathrm{target}}+ r_{\mathrm{collision}}+r_{\rm{dis}}+r_{\rm{track}}.
\end{equation}

\section{Inverse PPO Based on Adaptive Entropy}
In this section, we design a novel inverse PPO-based lightweight model to solve the fixed-wing UAV obstacle avoidance problem. The framework contains a lightweight backbone, an efficient strategy selection mechanism, and a new optimization objective function.

\subsection{Overview}
The overall framework is illustrated in Fig. \ref{fig:framework}. Our method is motivated by previous studies that examined the encoding of depth maps from multiple consecutive frames as state variables in reinforcement learning models for navigation and obstacle avoidance tasks. However, the storage and encoding of multiple frames of depth maps lead to high memory consumption and degrade the real-time performance of the system, rendering this approach unsuitable for fast-moving fixed-wing UAVs equipped with edge computing devices.

To address the challenge of increased memory consumption caused by stacking multiple depth maps, and to alleviate potential generalization issues that arise in depth inference, we incorporated a fine-tuned monocular depth estimation model proposed by \cite{bhat2023zoedepth}, which is proven to be reliable across a wide range of environments.
By fine-tuning this depth model for our specific application, we are able to
generate reliable enough depth maps for the following deep reinforcement learning module and at the same time 
reduce the computational burden of processing multiple depth frames.

Additionally, one of our primary objectives was to ensure that the proposed architecture sustained computational efficiency, particularly when deployed on edge devices with limited processing capabilities. To this end, we integrated \cite{ma2024rewrite}, a model specifically designed for efficient feature extraction, as part of our system architecture. Specifically, it improves feature extraction by performing element-wise multiplication between two linear transformation features, an operation inherently optimized for execution on Neural Processing Unit (NPU) architectures. NPUs are specifically optimized for matrix operations and parallel processing\cite{tan2022deep}, making them particularly suitable for operations involving intensive linear algebra computations. By leveraging the compatibility between the feature fusion mechanism and NPU hardware, we achieved both high performance and low power consumption, which are essential for edge computing environments. 

\subsection{Inferring Advantage Function}
The loss function of the traditional PPO algorithm is mainly based on the advantage function. To improve the universality of the algorithm, an inferring advantage function based on the reward function in 
Eq. \ref{eq:reward_function} is designed to make the algorithm a closed loop. 
\begin{equation}
\begin{array}{l}
{InA_\theta }\left( {{s_t},{a_t}} \right) \buildrel \Delta \over = {InQ_\theta }\left( {{s_t},{a_t}} \right) - {InV_\theta }\left( {{s_t}} \right)\\
 = {\mathbb E_{{s_t},{a_t}}}\left( {\sum\limits_l {{r_{t + l}}} } \right) - {\mathbb E_{{s_t}}}\left( {\sum\limits_l {{r_{t + l}}} } \right),
\end{array}
\end{equation}
where $\theta$ is the vector of policy parameters before the update, $InQ_\theta \left( {{s_t},{a_t}} \right)$ and ${InV_\theta }\left( {{s_t}} \right)$ are inferring action-value function and inferring value function, which can be obtained with reward function in (16). Therefore, ${InA_\theta }\left( {{s_t},{a_t}} \right)$ ia a inferring advantage function at timestep $t$.

In PPO algorithm, the importance sampling mechanism is used to control the updating range of the policy while optimizing the policy, so as to avoid the problem of drastic changes caused by excessive updating. However, during the training process, we usually use some old strategies that have been trained to collect samples, rather than using the latest strategies that are currently available. This leads to the problem of mismatch between the sample and the current policy, which is called "policy offset". In order to ensure the exploration of better solutions when the distribution gap between the two data is large, we no longer assume that the distribution of the old and new strategies is similar, and encourage the exploration of new strategies when the distribution difference between the old and new strategies is large. Therefore, this paper explores and utilizes data considering the distribution of old and new strategies
\begin{equation}
\begin{split}
\begin{array}{l}
{\mathbb E_{({s_t},{a_t}) \sim {\pi _\theta } }}\left[ {{InA_\theta }\left( {{s_t},{a_t}} \right)\nabla \log {\pi_\theta }\left( {a_t^n\left| {s_t ^ n} \right.} \right)} \right]\\
 = {\mathbb E_{({s_t},{a_t}) \sim {\pi_{\theta '}}}} \Bigg[ {{\frac{{{\pi_\theta }({a_t}\left| {{s_t}} \right.)}}{{{\pi_{\theta '}}({a_t}\left| {{s_t}} \right.)}}\frac{{{\pi_\theta }({s_t})}}{{{\pi_{\theta '}}({s_t})}} }} {InA_\theta }\left( {{s_t},{a_t}} \right)\\
\quad \quad \quad \quad \quad \quad \quad {{ \nabla \log {\pi_\theta }\left( {a_t^n\left| {s_t^n} \right.} \right)}} \Bigg],
\end{array}
\end{split}  
\end{equation}
where $\theta ' $ represents the vector of policy parameters after the update, ${\pi _\theta }$ and ${\pi _{\theta'}}$ are old and new strategies, respectively.

\subsection{Strategy Selection Mechanism}
In this section, we undertake a detailed examination of the factors that must be taken into account when selecting strategy mechanisms from two distinct perspectives.
\subsubsection{Balance Exploration and Exploitation}
The challenge of reinforcement learning lies in striking a balance between exploration and exploitation. An excess of exploration can result in situations where the algorithm fails to converge or converges slowly, whereas an excess of exploitation can lead to the disadvantage of local optimality. In the traditional PPO algorithm framework, it uses an importance sampling mechanism to train the model. More importantly, its assumed that the distributions of the training and learning models are consistent. However, the approach can lead to over-dependence of the data of the learning model on the merits of the training data, when the data trained by the intelligences are not picked for good strategies, making the learning success rate decrease. To solve the problem, we design a new strategy selection mechanism.

\subsubsection{Lowering Sensitivity to Prior Knowledge}
When viewed through the lens of prior knowledge, the efficacy of conventional PPO algorithms, along with other deep reinforcement learning techniques, is markedly influenced by the data accumulated in previous iterations. In an effort to mitigate this reliance on prior knowledge and drawing inspiration from maximum entropy methods, we have devised policy mechanisms that are not only robust and stable but also exhibit rapid convergence through the use of self-tuning.

\begin{definition}
The strategy entropy in the Markov process affects the balance between exploration and utilization, where for each state and action, the constraint is given as follows.
\begin{equation}
 H(a\left| s \right.) = {e^{\frac{{\sum\limits_t {{\gamma ^t}{r_t}} }}{r_{e}}}}\left( { - \sum\limits_a {\pi_\theta (s,a)} \log \pi (s,a)} \right),   
\end{equation}
\end{definition}
where $r_{e}$ denotes an expected value of reward. 

When the aforementioned entropy $\mathcal H $ is higher, there is a greater propensity for utilization. Conversely, when entropy is lower, there is a greater propensity for exploration.

\begin{remark}
The previously mentioned strategy of entropy allows us to effectively address the challenge of balancing exploration and exploitation. However, in light of the necessity for simplified implementation and reduced computational complexity in engineering, there is a clear need for the development of more sophisticated entropy operators.
\end{remark}

To make it learn under conditions that increase its success rate, we design a more generalized strategy entropy mechanism.
\begin{equation}
 H(a\left| s \right.) = \frac{{\sum\limits_t {{M_s}} }}{{Batch}}\left( { - \sum\limits_a {\pi_\theta (s,a)} \log \pi_\theta (s,a)} \right),
\end{equation}
where $M_s$ is the total number of successes in a ${Batch}$, ${Batch}$ represents a set of data samples used when updating a policy or value function.

\begin{lemma}{\label{lemma1}}
$ H(\pi (s,a))$ is $\eta$-smooth, equipped with the Taylor’s theorem, we have such that
\begin{equation}
\begin{array}{l}
{\left\| {\nabla  H(\pi_\theta (s,a)) - \nabla  H(\pi_\theta (s',a))} \right\|_\infty }\\
\le \eta {\left\| {\pi_\theta (s,a) - \pi_\theta (s',a)} \right\|_\infty },
\end{array}
\end{equation}
where $\eta$ is a coefficient.
\end{lemma}

\begin{theorem}
For any $ k < N, k \in \mathbb{N}$, the entropy can be used as the sample mean as follows,
\begin{equation}
\begin{split}
\begin{array}{l}
 H(\pi_\theta (s,a)) = \frac{1}{T}{\sum\limits_{i = 0}^{T - 1} {\Bigg[ {\left( {T - 1} \right)V^m{C_k}}}}\\
\quad \quad \quad \quad \quad \quad \quad \quad \quad {{{{{{\left\| {\pi_\theta (s,a) - \pi_\theta (s',a)} \right\|}^m}} }\Bigg]} ^\lambda },
\end{array}
\end{split}
\end{equation}
where ${C_k} = {\left[ {\frac{{k!}}{{\left( {k + 1 - \left[ \lambda  \right]} \right)!}}} \right]^\lambda }$ and ${V^m} = {\raise0.7ex\hbox{${{\pi ^{\frac{m}{2}}}}$} \!\mathord{\left/
 {\vphantom {{{\pi ^{\frac{m}{2}}}} {(\frac{m}{2} + 1)!}}}\right.\kern-\nulldelimiterspace}
\!\lower0.7ex\hbox{${(\frac{m}{2} + 1)!}$}}$ is the volume of the unit ball $\mathcal B(0, 1)$ in $\mathbb R^m$. And it holds that
\begin{equation}
\mathop {\lim }\limits_{N \to \infty } H_N^k(\pi_\theta (s,a)) = H(\pi_\theta ^{\mathcal L} (s,a)),
\end{equation}
where $\pi ^{\mathcal L} (s,a)$ denotes the Lebesgue measure. 
\end{theorem}

\begin{proof}
Equipped with Lemma \ref{lemma1}, let we have ${\pi ^*} = \arg \mathop {\max }\limits_{\pi  \in \Pi } \mathcal H(\pi_\theta (s))$
\begin{equation}
\begin{array}{l}
 H(\pi_\theta ^* (s,a)) -  H(\pi_\theta (s,a))\\
 \le \kappa \exp ( - T\eta ) + 2\beta \sigma  + \zeta ,
\end{array}
\end{equation}
Note that $ \mathcal H(s)$ is finite if $\pi_\theta (s)$ is of bounded support. Indeed, consider the imposed smoothing on $ \mathcal H(\pi_\theta (s,a))$,we have 
\begin{equation}
\begin{array}{l}
 H(\pi_\theta (s,a)) \ge \mathop {\max }\limits_{\pi_\theta  \in \Pi } H(\pi_\theta (s,a)) - \sigma,
\end{array}
\end{equation}
and
\begin{equation}
\begin{array}{l}
{\left\| {\nabla H(\pi_\theta (s,a))} \right\|_\infty } \le \kappa  = {e^{\frac{{\sum\limits_t {{\gamma ^t}{r_t}} }}{{Exp(R)}}}},
\end{array}
\end{equation}
where $T \ge 10\zeta \kappa \log 10\zeta $. 

Hence, $ H(\pi_\theta (s,a))$ is tend to the support of $\pi_\theta ^{\mathcal L} (s,a)$. This concludes the proof.
\hfill $\square$
\end{proof}

\subsubsection{Learning Objective}
In the traditional PPO implementations, the process of training is influenced by a fixed hyperparameter which determines the exploration magnitude. In this paper, a new PPO method with an inferring reward mechanism and adpative entropy is introduced, which incorporates a dynamic scaling of the entropy coefficient based on the recent return obtained by the agent. 
Based on the above discussion, the final loss function can be written in the following form,

\begin{small}
\begin{equation}
\begin{split}
L^{CLIP}_{\theta} = \mathbb{E}_{({s_t},{a_t}) \sim {\pi _\theta ' } } \Bigg[ \min \left( 
    {{\frac{{{\pi_\theta }({a_t}\left| {{s_t}} \right.)}}{{{\pi_{\theta '}}({a_t}\left| {{s_t}} \right.)}}\frac{{{\pi_\theta }({s_t})}}{{{\pi_{\theta '}}({s_t})}} }} InA_\theta (s_t, a_t), \right. \\
    \left. clip \left( \frac{\pi_\theta (a_t \mid s_t)}{\pi_{\theta '}(a_t \mid s_t)}, 1 - \varepsilon, 1 + \varepsilon \right) InA_\theta (s_t, a_t) 
    \right) \Bigg],
\end{split}
\end{equation}
\end{small}

\begin{equation}
\begin{split}
{L^{VF}_{\theta}} = \left( InV_{\theta}(s_t) - InV_t^\text{target} \right)^2,
\end{split}
\end{equation}

\begin{equation}
\begin{split}
{L^{ENT}_{\theta}} = \mathbb{E}_{({s_t},{a_t}) \sim {\pi _\theta } } \left[ \mathcal{H}(\pi_\theta(a_t|s_t)) \right],
\end{split}
\end{equation}

\begin{equation}
\begin{split}
{L^{Inverse}} = L^{CLIP} - w_1 L^{VF} + w_2 L^{ENT},
\end{split}
\end{equation}
where $w_1$ and $w_2$ are hyperparameters that are used to regulate the effects of value loss and entropy.

\section{Experimental Validation}
To evaluate its effectiveness, three experimental setups are designed in this section. {  
First, we design an ablation experiment to separately demonstrate the impact of the designed reward function and the update mechanism. Second, we demonstrate the superiority of the proposed method through comparisons with other deep reinforcement learning algorithms. Finally, a hardware-in-the-loop simulation experiment is conducted to verify the deployment capability of the proposed framework on edge devices, while also comparing it with classic sample-based methods.}
\begin{table}[htbp]
\centering
\caption{Common Parameter Settings for PPO and Proposed Algorithm}
\begin{tabular}{cc}
\hline
\textbf{Parameter}        & \textbf{Value} \\ \hline
Air Speed (m/s)           & 30         \\ 
Depth Map Size ($\mathcal{}{H},W$)            & 224, 224         \\ 
Reward Term Weight (${C_1}, C_2, C_3, C_4$)    & 30, -30, 0.5, 1.0 \\ 
Flying Distance Cap $d_{max}$ (m)              & 1300 \\
Learning Rate             & 0.0003         \\ 
Gamma (\(\gamma\))        & 0.95           \\ 
Clip Range (\(\epsilon\)) & 0.3            \\ \hline
K Epochs                  & 2              \\ 
Batch Size                & 2048             \\ 
Value Loss Coefficient(${w_1}$)    & 0.5            \\ 
Entropy Loss Coefficient(${w_2}$)       & 0.1           \\ 
Max Timesteps Per Episode             & 60            \\ 
Max episodes             & 3000            \\ 
State Dimension           & 256            \\ 
Action Dimension          & 8              \\ \hline
\end{tabular}
\label{tab:PPO_Settings}
\end{table}
\begin{figure}[htp]
    \centering
    \includegraphics[width=1.0\linewidth]{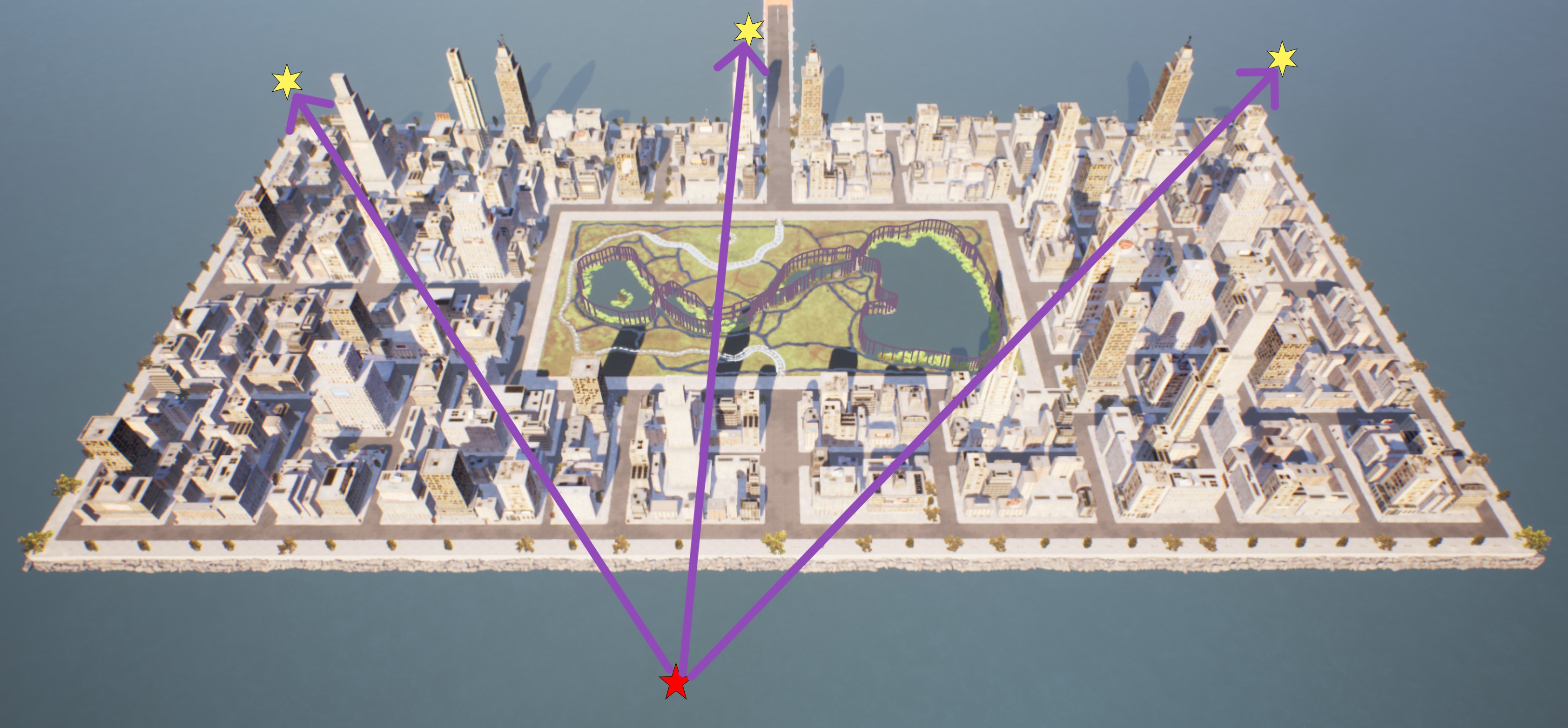}
    \caption{Training flight paths. The yellow six-pointed stars represent the targets, the red star indicates the fixed-wing UAV's take-off position, and the purple line represents the expected flight trajectory.}
    \label{fig:train_lines}
\end{figure}
\begin{figure*}[t]
    \centering
    \begin{minipage}[b]{0.45\textwidth}
        \centering
        \includegraphics[width=\textwidth]{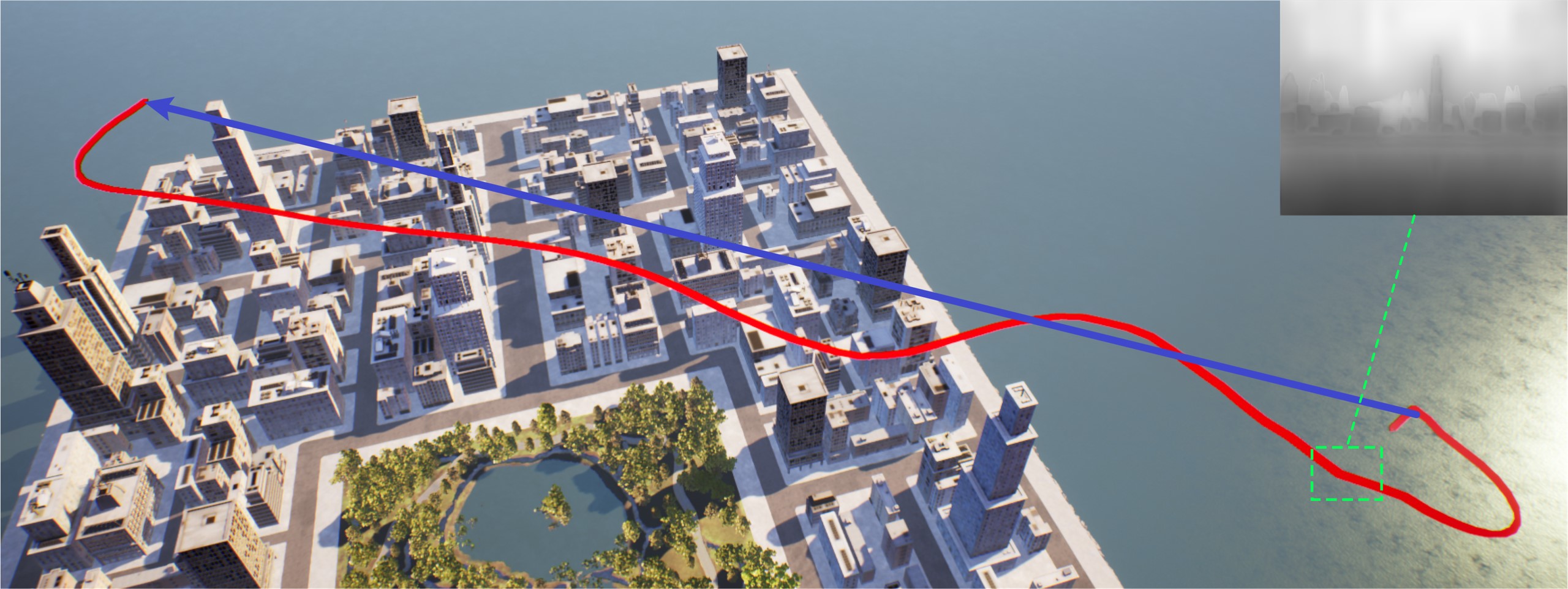}
        (a)
    \end{minipage}
    \hfill
    \begin{minipage}[b]{0.45\textwidth}
        \centering
        \includegraphics[width=\textwidth]{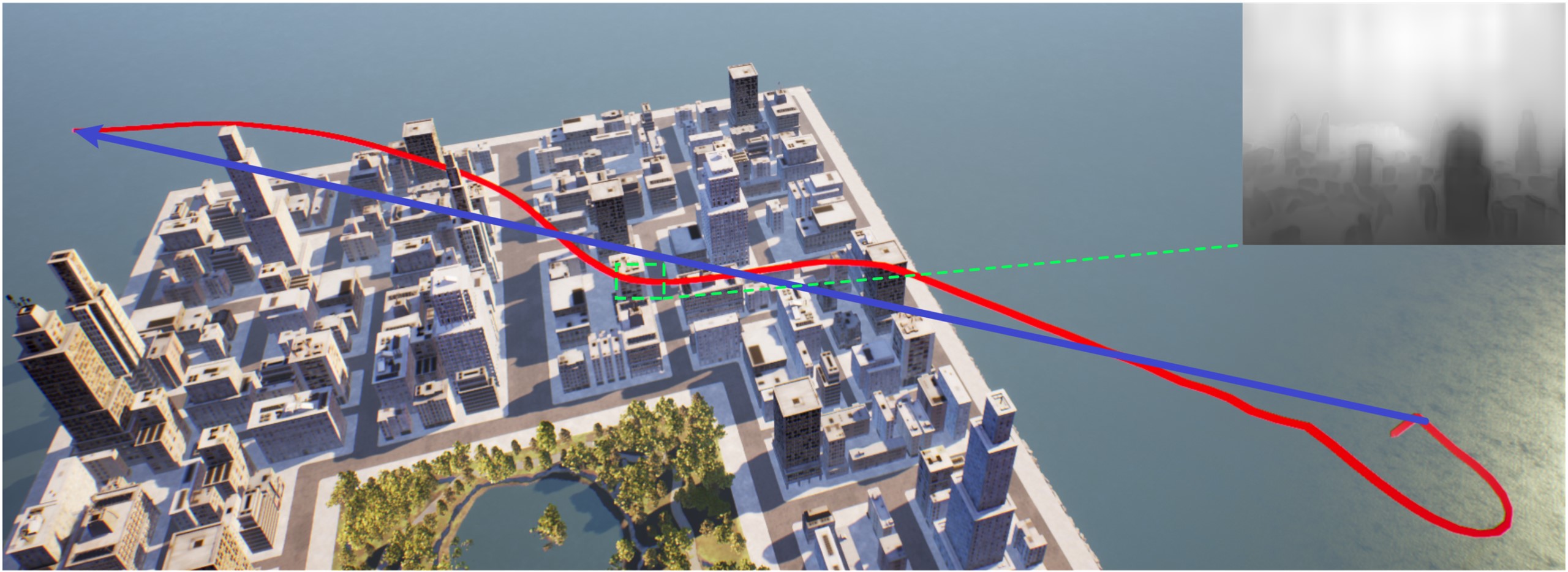}
        (b)
    \end{minipage}

    \vspace{0.5cm} 

    \begin{minipage}[b]{0.45\textwidth}
        \centering
        \includegraphics[width=\textwidth]{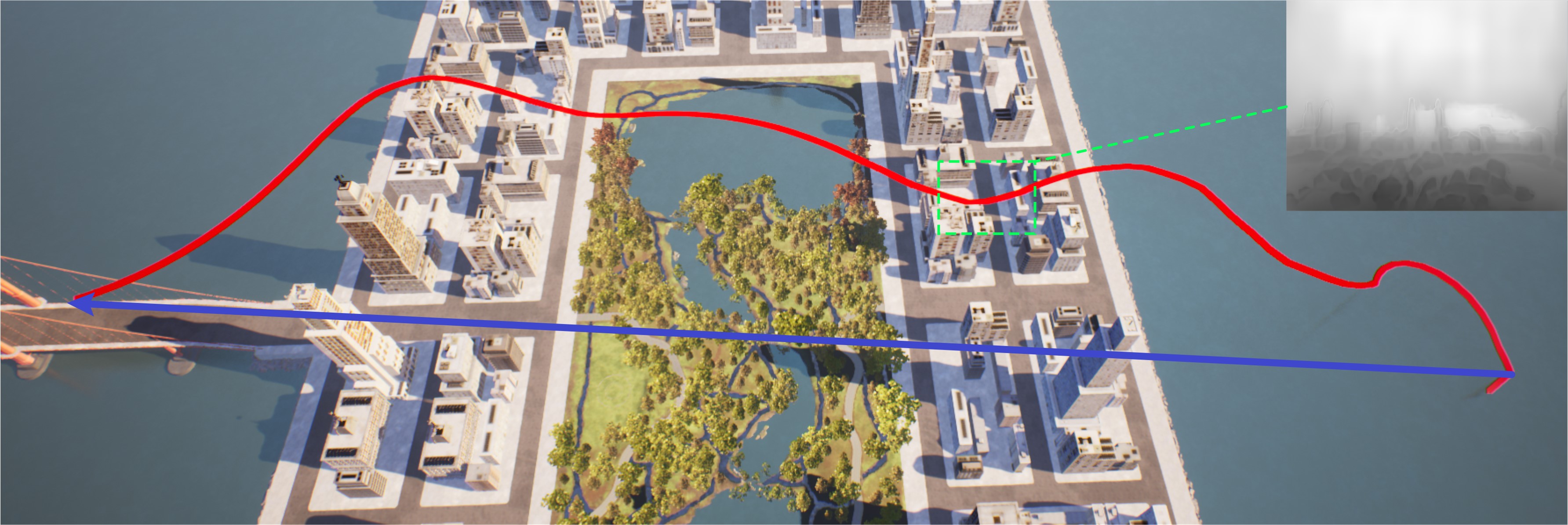}
        (c)
    \end{minipage}
    \hfill
    \begin{minipage}[b]{0.45\textwidth}
        \centering
        \includegraphics[width=\textwidth]{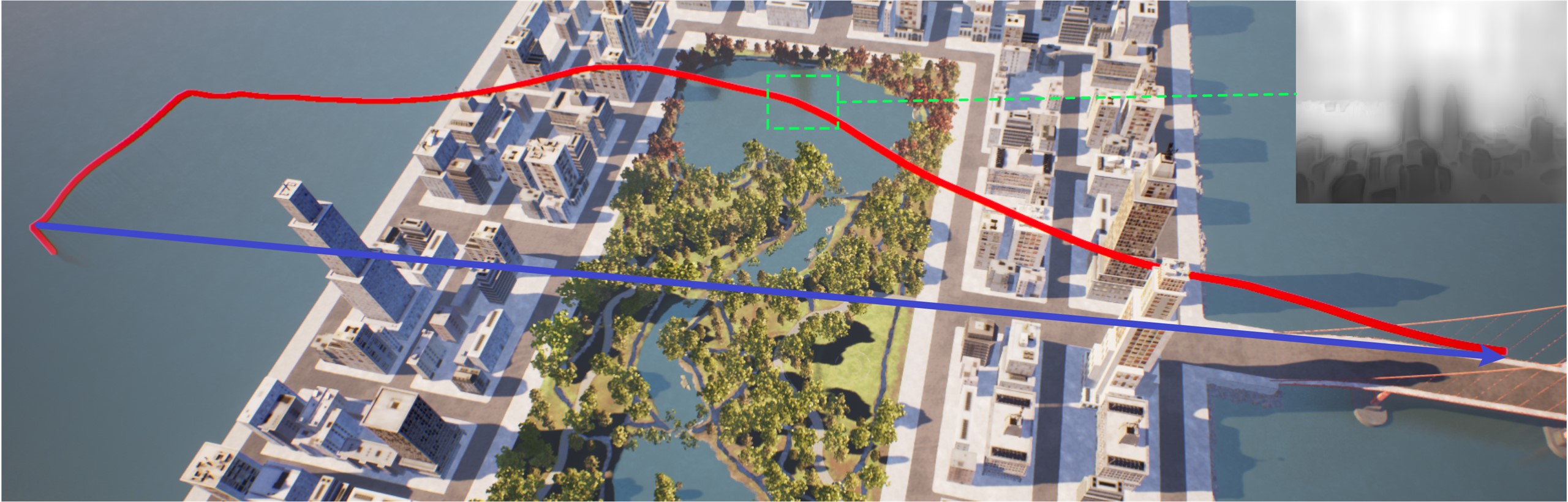}
        (d)
    \end{minipage}

    \vspace{0.5cm}

    \begin{minipage}[b]{0.45\textwidth}
        \centering
        \includegraphics[width=\textwidth]{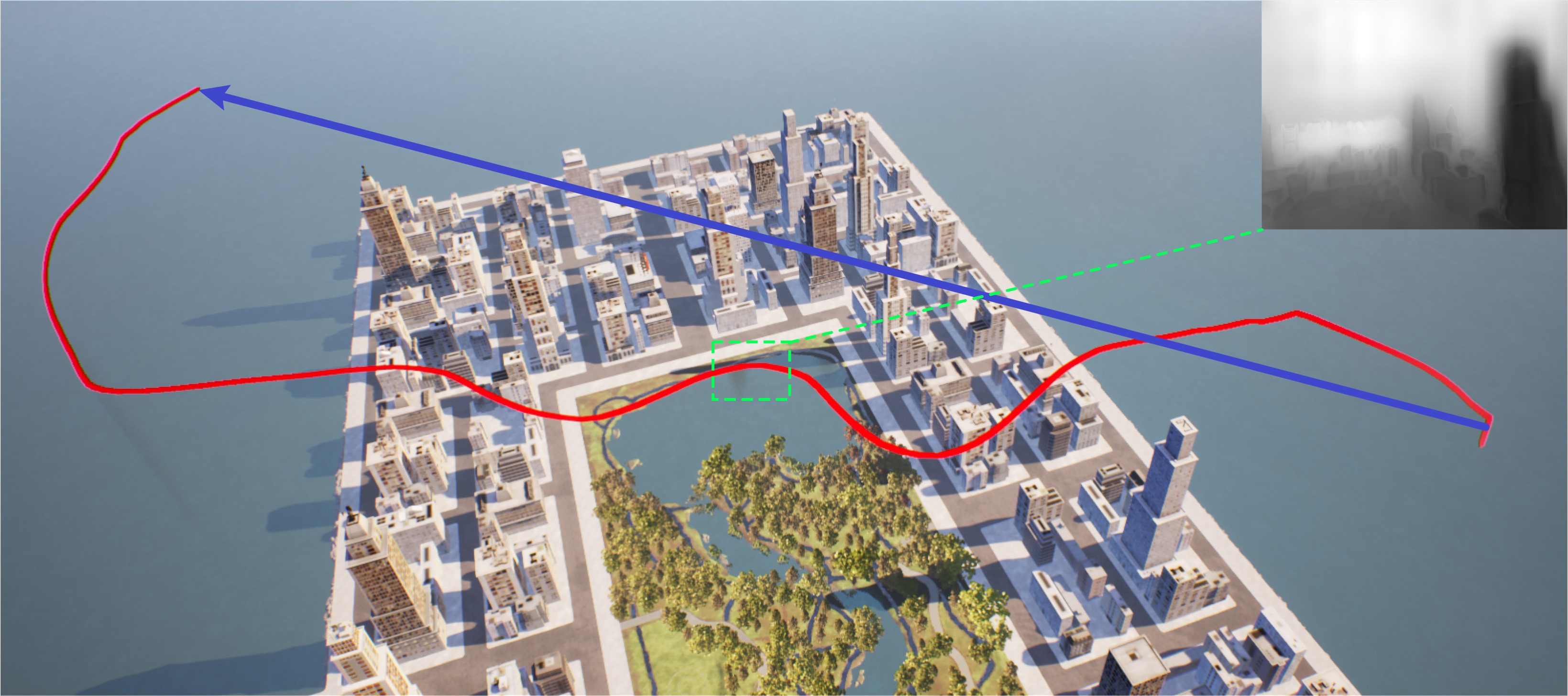}
        (e)
    \end{minipage}
    \hfill
    \begin{minipage}[b]{0.45\textwidth}
        \centering
        \includegraphics[width=\textwidth]{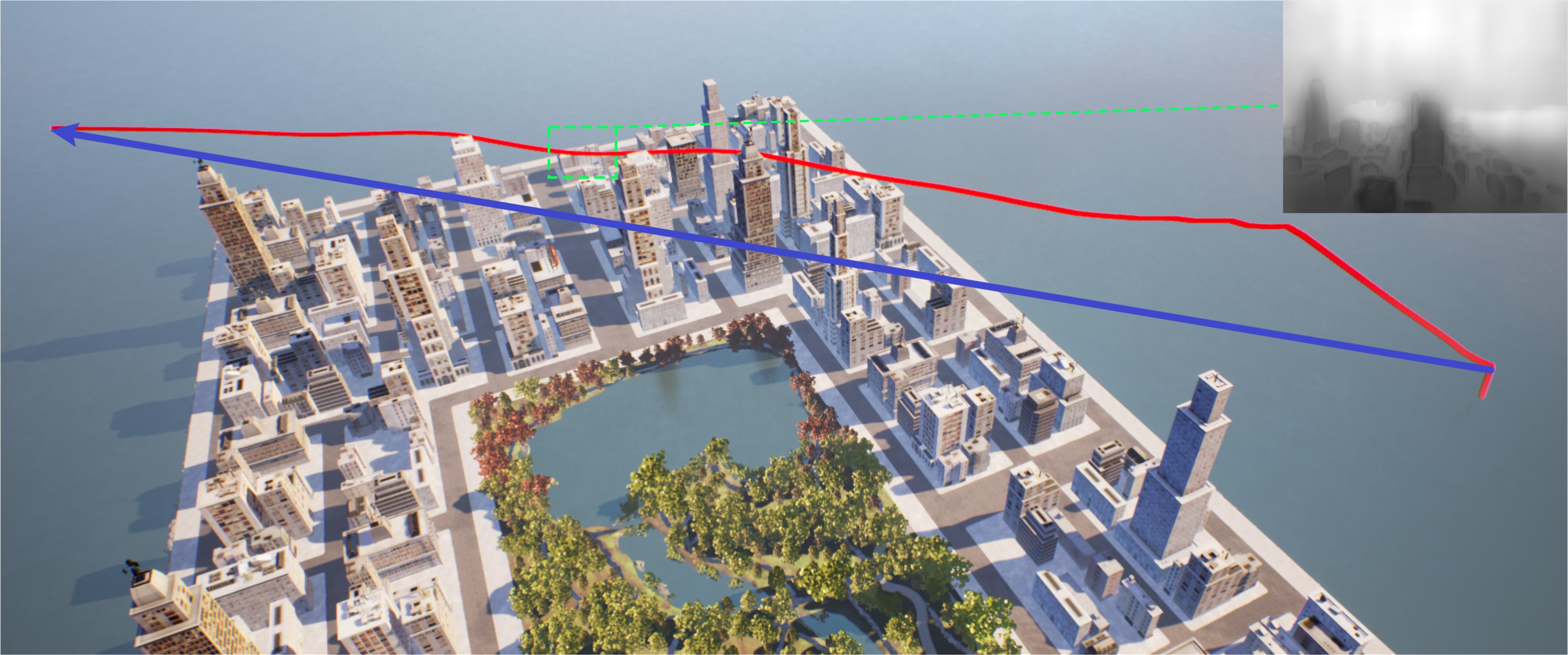}
        (f)
    \end{minipage}
    
    \caption{The comparison of the impact of different reward functions on obstacle avoidance flight trajectories. The red solid lines represent the flight trajectories of the fixed-wing UAV generated by the decision-making process of deep reinforcement learning (DRL) algorithms. The blue solid lines with arrows represent the expected flight trajectories, which point from the take-off points toward the target points. The green dashed lines represent the inferred depth map during obstacle avoidance maneuvers. (a), (c), and (e) show the obstacle avoidance trajectories generated by the model that only uses $r_{\rm{dis}}$, while (b), (d), and (f) show the obstacle avoidance trajectories produced by the model trained using the proposed reward function.}
    \label{fig:six_images}
\end{figure*}
\subsection{Training Settings}
    We conduct the training on a machine equipped with an Intel Xeon E5-2678 V3 CPU and two NVIDIA RTX 3090 GPU. A high-fidelity simulator, AirSim\cite{airsim2017fsr} build on Unreal Engine (UE), is employed to build the different environments and provide data including RGB images captured by its camera and fixed-wing UAV’s position. The fixed-wing UAV's dynamics model is provided by JSBSim\cite{berndt2004jsbsim}, an open-source platform widely regarded for its high accuracy in modeling aerodynamics and flight physics. The specific model used in this work is the Skywalker X8, a popular choice for its stability and versatility in various flight scenarios. The neural network models are established using the PyTorch framework.

We conduct the training of the proposed method with the parameters shown in Table \ref{tab:PPO_Settings} within a 1000m by 600m rectangular urban environment constructed using UE. In the experiments, target points are randomly selected from three predefined flight paths. As shown in Fig. \ref{fig:train_lines}, varying numbers of obstacles are distributed along the three flight paths. The variation in obstacle density across the routes simulates the challenge of avoidance faced by fixed-wing UAVs in environments with different levels of obstacle density. Image data are collected using a simulated camera provided by AirSim, generating color images with a resolution of 480×640 for the depth estimation module.

\subsection{Ablation Studies}
In this section, we study the importance of various design modules in our framework.
\subsubsection{Inferred reward function}
We evaluate the contribution of the designed reward function to the smoothness and stability of flight trajectories. First, we employ only the distance reward function, assigning it a weight of $C_3$=1 and referring to this configuration as the distance model. Subsequently, we compare the smoothness of the flight trajectories generated by the distance model and the proposed model along three predefined flight paths. As illustrated in Fig. \ref{fig:six_images} (b), (d), and (f), the stability that observed in the proposed model's trajectories can be attributed to the carefully designed reward function, which balances multiple factors such as obstacle avoidance and trajectory smoothness. In contrast, the trajectories generated by the model using only $r_{\rm{dis}}$ exhibit more abrupt directional changes, as highlighted by the jagged red solid lines in Fig. \ref{fig:six_images} (a), (c), and (e). Such rapid course corrections can impose additional strain on the fixed-wing UAV's control system, which can lead to potential instability, particularly in complex urban environments. By incorporating a smoothness criterion into the reward structure, the proposed model effectively reduces the necessity for drastic course adjustments, enabling the UAV to follow a more fluid and consistent trajectory. 
\begin{table*}[htbp]
\centering
\caption{Task Completion Results of Different Obstacle Avoidance Strategies in Different Scenes}
\begin{tabular}{cc|c|c}
\toprule
\multicolumn{1}{c}{\textbf{Success Rate (\%, $\downarrow$)}} & \multicolumn{1}{c|}{\textbf{Scene 1: City}} & \multicolumn{1}{c|}{\textbf{Scene 2: Line-cruising}} & \multicolumn{1}{c}{\textbf{Scene 3: Valley}}\\
\midrule
\textbf{Proposed} & \textbf{86.0} & \textbf{80.0} & \textbf{74.0} \\
\textbf{PPO} & 82.0 & 76.0 & 69.0 \\
\textbf{TRPO} & 80.0 & 74.0 & 68.0 \\
\textbf{A3C} & 78.0 & 72.0 & 66.0 \\
\textbf{DQN} & 77.0 & 70.0 & 64.0 \\
\textbf{DDPG} & 76.0 & 68.0 & 62.0 \\
\bottomrule
\end{tabular}
\label{tab:results}
\end{table*}
\begin{figure}[htp]
    \centering
    \includegraphics[width=1.0\linewidth]{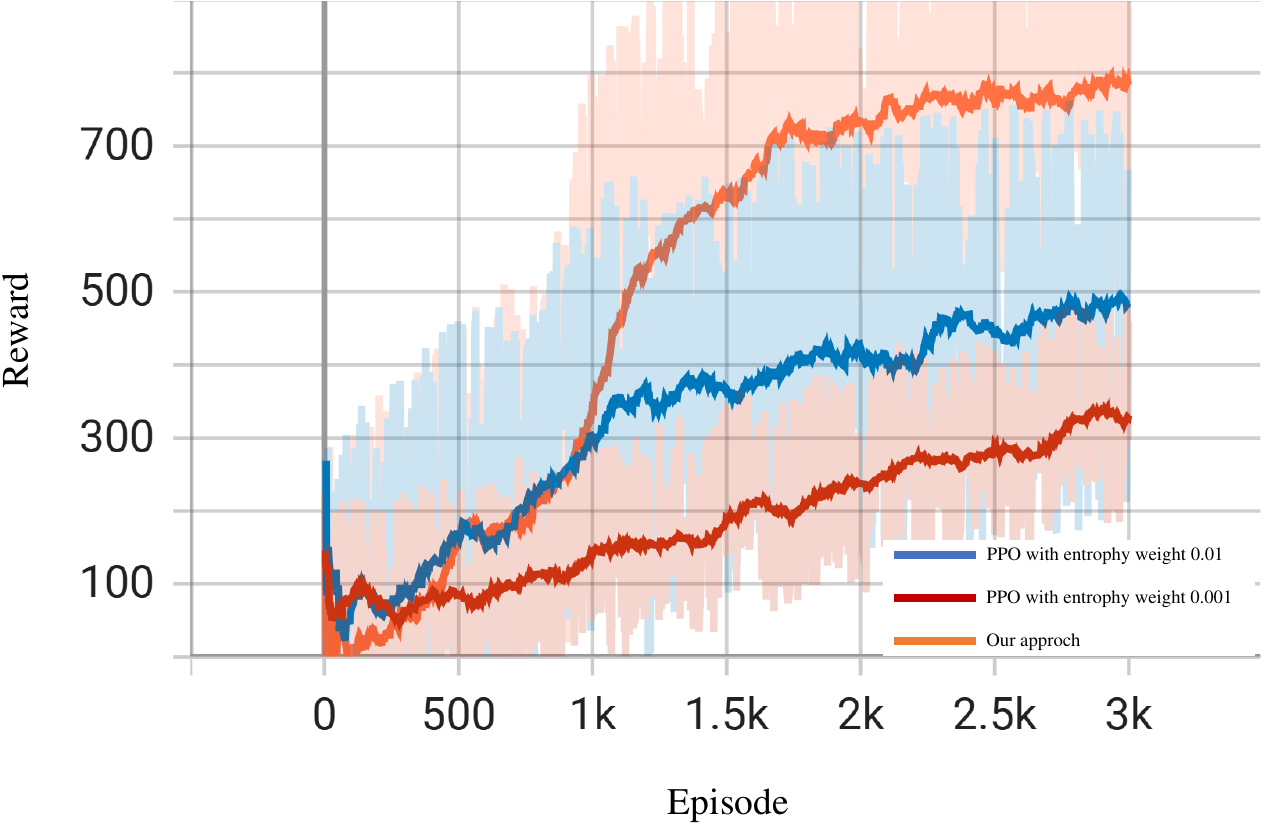}
    \caption{Training  cumulative rewards comparison. The solid lines
    represent the average rewards of our algorithms and baselines per episode, while the shaded areas indicate the variability in the reward accumulation for each method.}
    \label{fig:reward}
\end{figure}
This improvement in flight stability is critical in real-world scenarios, where maintaining smooth trajectories helps minimize energy consumption and ensures safer navigation through dynamic and uncertain environments. Thus, the results clearly demonstrate the superiority of the proposed reward function in producing smoother and more stable flight paths for fixed-wing UAVs, ultimately enhancing overall flight performance in obstacle-rich environments.

\subsubsection{Adaptive entropy}
To evaluate the impact of the proposed Adaptive Entropy on the obstacle avoidance tasks, we design a strategy comparison experiment. We respectively train the PPO models with entropy weights of 0.01 and 0.001, and then compared their rewards per episode during training with the proposed method. The rewards obtained during training are shown in Fig. \ref{fig:reward}. The lower weight model shows a gradual improvement in performance, starting with relatively low rewards and steadily increasing as the training progresses, indicates moderate variability, suggesting consistent learning behavior. The higher weight, while starting at a lower initial reward, exhibits a steady increase over time, indicating higher variability in performance, particularly in the earlier stages of training. Our proposed method demonstrates the fastest learning curve, with cumulative rewards increasing rapidly in the early episodes. By the end of the training, it converges at the highest cumulative reward. The shaded area around the orange curve is relatively narrow, indicating low variability and suggesting that the proposed method is more stable and consistent across different episodes.

These results indicate that the proposed method outperforms PPO model with different entrophy weights in terms of cumulative rewards, demonstrating its effectiveness in navigating the fixed-wing UAVs through environments with obstacles. Additionally, the faster convergence of the proposed method shows its potential for quicker policy learning, making it suitable for pratical applications where rapid adaptation is critical.

\subsection{Policy Comparison}
To evaluate the effectiveness of our proposed fixed-wing UAV obstacle avoidance method, we conducted tests in three distinct scenarios named as Scene 1 (City), Scene 2 (Line-cruising), and Scene 3 (Valley). In the first scene place, the fixed-wing UAV in an urban environment, where it have to fly through densely packed buildings and avoid structural obstacles as illustrated in Fig. \ref{fig:Compare_images}(a). In the second scene, simulating a power line inspection in mountainous terrain, the fixed-wing UAV’s primary task is to avoid obstacles such as mountainous ridges and power poles while maintaining proximity to power lines as depicted in Fig. \ref{fig:Compare_images}(b). In the third scene, the fixed-wing UAV face a desert canyon with dynamic terrain changes, where the algorithm has to account for steep ascents and descents while avoiding natural formations such as cliffs and ridges as shown in Fig. \ref{fig:Compare_images}(c). These scenarios represent varying levels of environmental complexity, incorporating differences in obstacle types, numbers, and distributions.

We compare our proposed method with several established reinforcement learning algorithms, including PPO, TRPO, A3C, DQN, and DDPG. All algorithms are tested in the same initial simulation environment, with repeated trials in each scenario to ensure robust results. In each scenario, the agent's task is to fly from a starting position to a target without colliding with any obstacles. The performance of each algorithm is measured using the task completion rate (Success Rate, \%), which represents the percentage of trials in which the agent successfully completed the task. Each test is repeated 100 times per scenario to minimize the impact of randomness on the results. Table \ref{tab:results} presents the task completion rates for our proposed method and the other strategies across the three scenarios. The results demonstrate that our proposed approach outperforme the baseline algorithms in all scenarios. Specifically, in the City and Line-cruising scenarios, our method achieved task completion rates of 86.0\% and 80.0\%, respectively, which are higher than those of the other algorithms. In the more complex Valley scenario, our approach still maintain a strong performance with a 74.0\% success rate.

In comparison, the PPO algorithm achieves 82.0\%, 76.0\%, and 69.0\% in the three scenarios, which is slightly lower than our method. Other algorithms such as TRPO, A3C, DQN, and DDPG perform relatively worse, with success rates below 80.0\% in all scenarios. Notably, in the Valley scenario, DDPG exhibit the lowest success rate of 62.0\%.The experimental results indicate that our proposed method is capable of effectively handling different levels of obstacle complexity across various environments, achieving higher task completion rates than existing reinforcement learning strategies. We hypothesize that the superior performance of our method, particularly in complex environments, can be attributed to its adaptive strategy optimization and efficient exploration mechanism. While the performance of all algorithms is comparable in simpler scenarios, such as Line-cruising, our method shows a significant advantage in more complex scenarios like City and Valley.
\begin{figure}[htp]
    \centering
    \includegraphics[width=1.0\linewidth]{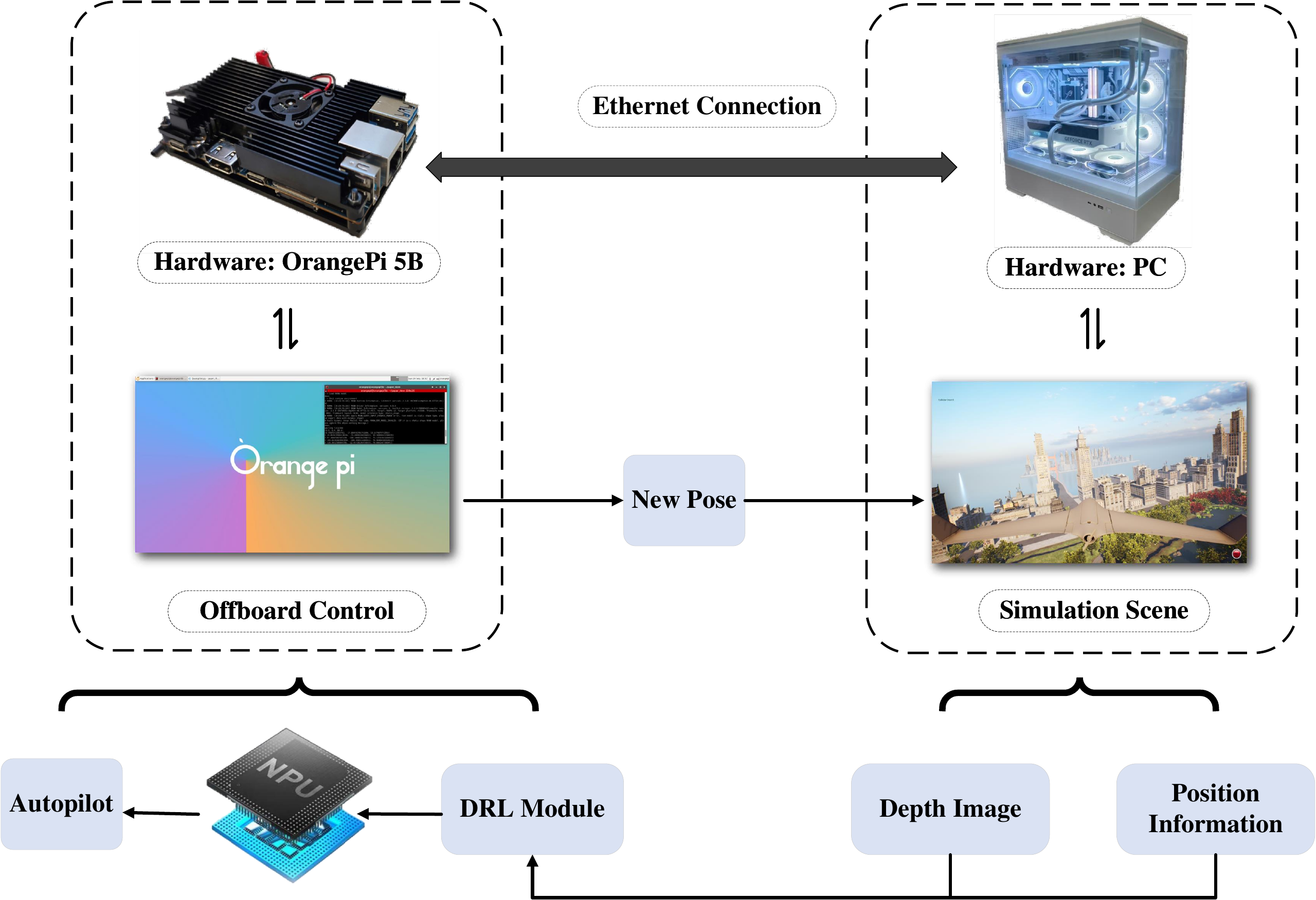}
    \caption{Hardware-in-the-loop (HIL) platform structure. The platform consists of hardware components, where the Orange Pi 5B and PC communicate via an Ethernet connection. The software components include the offboard control module and simulation scene module, used for control and decision validation in a simulated environment.}
    \label{fig:platform}
\end{figure}
\subsection{Hardware-in-the-loop Simulation}
A hardware-in-the-loop simulation experiment is conducted to demonstrate the deployability of the proposed algorithm. Additionally, a comparison with sample-based algorithms \cite{paranjape2015motion} is made to validate the performance of the proposed approach. The simulation platform consists of a computer equipped with an Intel i5-13600KF CPU and an NVIDIA RTX 4070Ti SUPER GPU, acting as the primary simulation unit. The onboard edge computing platform, the OrangePi 5B, equipped with a Rockchip RK3588s processor and a Neural Processing Unit (NPU), is used to execute real-time inference of the trained model. The model is initially trained and converted to the RKNN format using the RKNN-Toolkit2 (v2.1.0) and deployed via the Python API. The experimental platform is shown in Fig. \ref{fig:platform} and validation scenes are the same as the aforementioned \textit{Policy Comparison} tests.

\begin{figure}[htp]
    \centering
    \begin{minipage}[b]{0.48\textwidth}
        \centering
        \includegraphics[width=\linewidth]{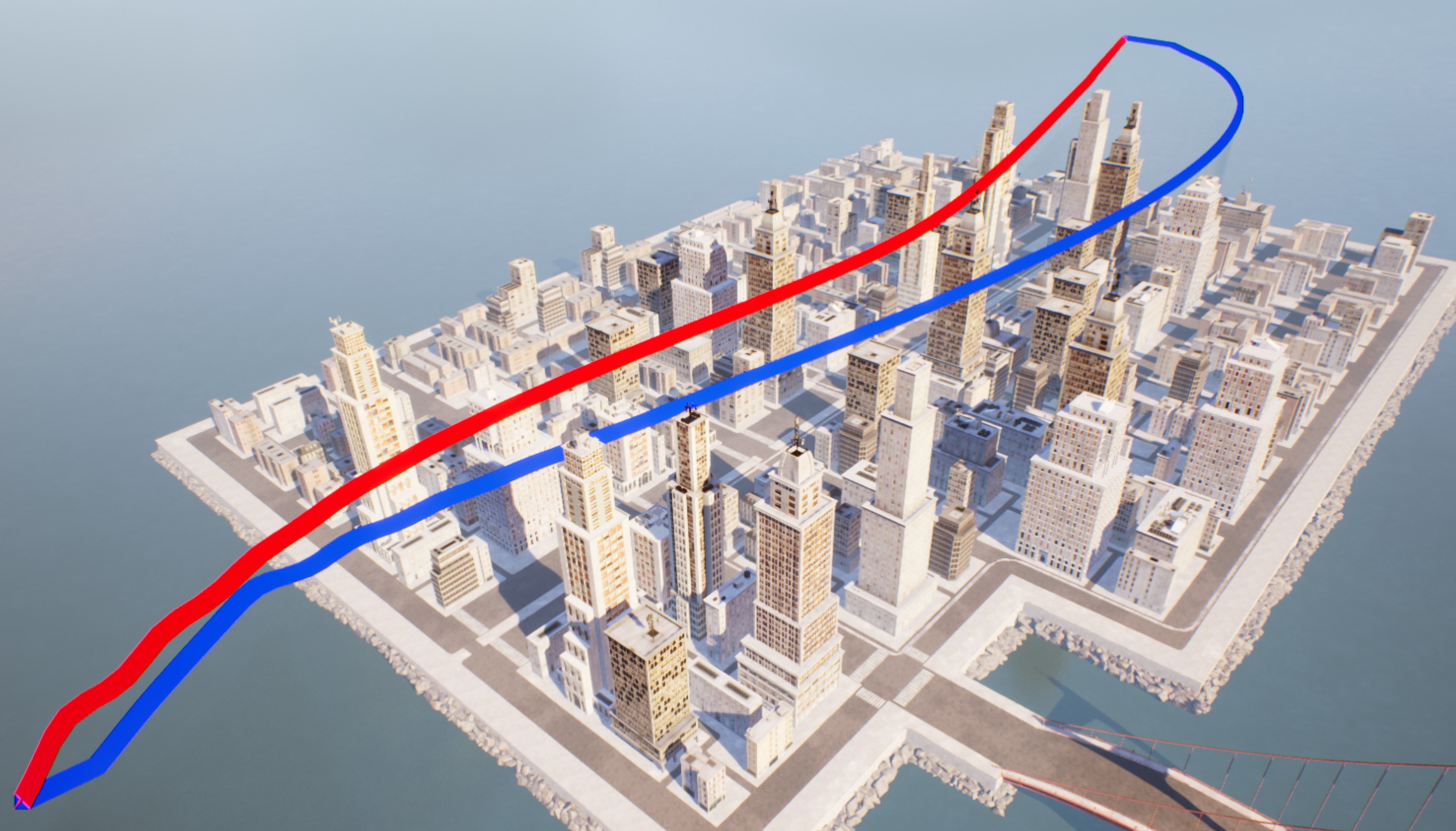} 
        \caption*{(a) Scene I} 
        \label{fig:C_sub1}
    \end{minipage}
    \hfill
    \begin{minipage}[b]{0.48\textwidth}
        \centering
        \includegraphics[width=\linewidth]{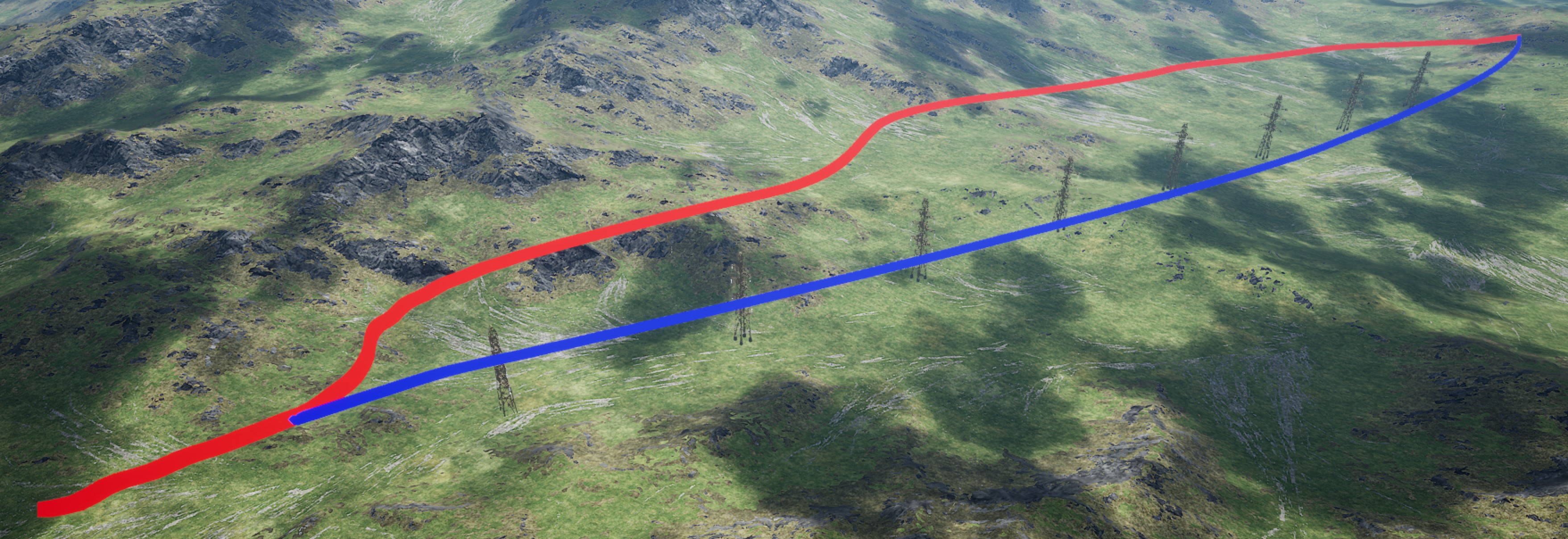} 
        \caption*{(b) Scene II}
        \label{fig:C_sub2}
    \end{minipage}
    \hfill
    \begin{minipage}[b]{0.48\textwidth}
        \centering
        \includegraphics[width=\linewidth]{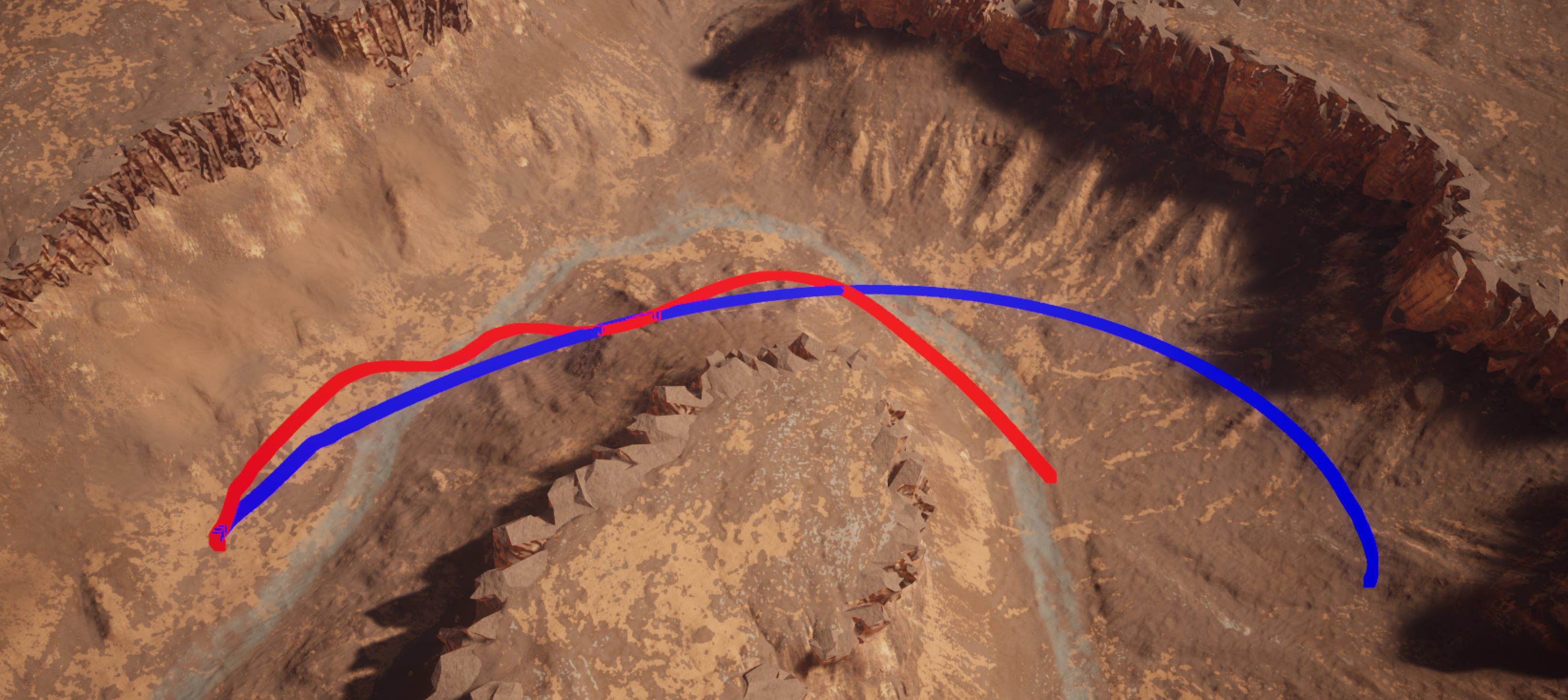} 
        \caption*{(c) Scene III}
        \label{fig:C_sub3}
    \end{minipage}

    \caption{HIL comparison between proposed method and sample based method in different scenes. The red line represents the flight trajectory generated by our proposed method, while the blue line represents the sample based method.}
    \label{fig:Compare_images}
\end{figure}

\subsubsection{Methods Comparison}
The comparison of flight trajectories is presented in Fig. \ref{fig:Compare_images}, showcasing the performance of both algorithms across the three environments. In the first scenario, with dense obstacles, the proposed algorithm produces a smoother and shorter path, as illustrated in Fig. \ref{fig:Compare_images}(a), while the sample-based method tends to choose regions with fewer obstacles for its flight path. As depicted in Fig. \ref{fig:Compare_images}(b), where the obstacles are sparsely distributed, the sample-based method is able to generate a trajectory that is closer to the expected flight path compared to the proposed method. In the third scenario, although the sample-based method generates a smoother trajectory, its turning capability is limited due to sampling only within the visible area, preventing it from navigating through the canyon, as shown in Fig. \ref{fig:Compare_images}(c). In contrast, the proposed algorithm has a broader range of options, significantly improving its turning ability.

\subsubsection{Quantitative Analysis}
The quantitative analysis of the results is presented in Fig. \ref{fig:qu_Analists}, where the X and Y coordinate distributions of the flight trajectories are plotted for each scene. The proposed algorithm's trajectory (solid line) is compared with the expected trend (dashed line) across all three environments. As shown in Fig. \ref{fig:qu_Analists}(a), the deviations between the proposed trajectory and the expected trend were minimal, indicating a close adherence to the optimal flight path in less challenging environments. In Scene II, the proposed method produced a noticeably smoother trajectory with fewer oscillations, especially in the Y coordinate distribution, demonstrating its superiority in densely cluttered environments, as depicted in Fig. \ref{fig:qu_Analists}(b). This trend contrasts sharply with the results shown in Fig. \ref{fig:qu_Analists}(c), where the greatest divergence between the two methods is evident. The proposed algorithm's ability to execute sharp turns and navigate narrow spaces enabled it to successfully complete the flight path, in contrast to the sample-based method, which struggled in this scenario.

Overall, the experimental results show that the proposed algorithm outperforms the sample-based method in more complex environments (Scene II and Scene III), particularly in terms of trajectory smoothness and path length. While the sample-based method performs better in simpler environments with sparse obstacles (Scene I), the proposed algorithm demonstrates greater adaptability and robustness in challenging, real-world scenarios. These findings suggest that  the proposed algorithm's ability to make quick decisions in real-time, demonstrating its potential for deployment in real-world edge-based fast moving fixed-wing UAV systems where obstacle avoidance is critical.

\section{Conclusion}
In this paper, we present a lightweight DRL framework that leverages inferred single-frame depth maps as input and employs a lightweight network architecture to address the obstacle avoidance challenges of high-speed fixed-wing UAVs.
Our framework incorporates an inferring reward function to address the stability 
\begin{figure}[t]
    \centering
    \begin{minipage}[b]{0.48\textwidth}
        \centering
        \includegraphics[width=\linewidth]{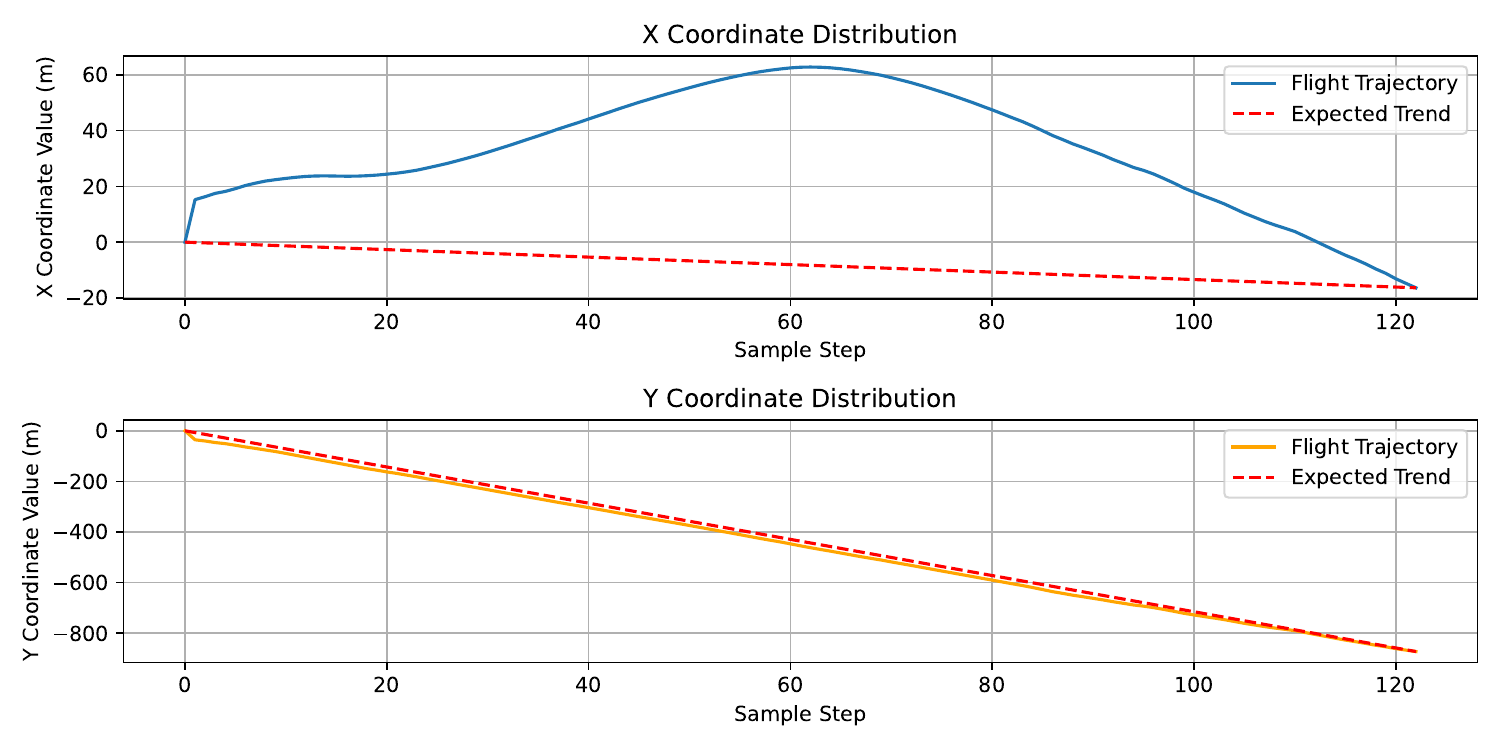} 
        \caption*{(a) Scene I} 
        \label{fig:sub1}
    \end{minipage}
    \hfill
    \begin{minipage}[b]{0.48\textwidth}
        \centering
        \includegraphics[width=\linewidth]{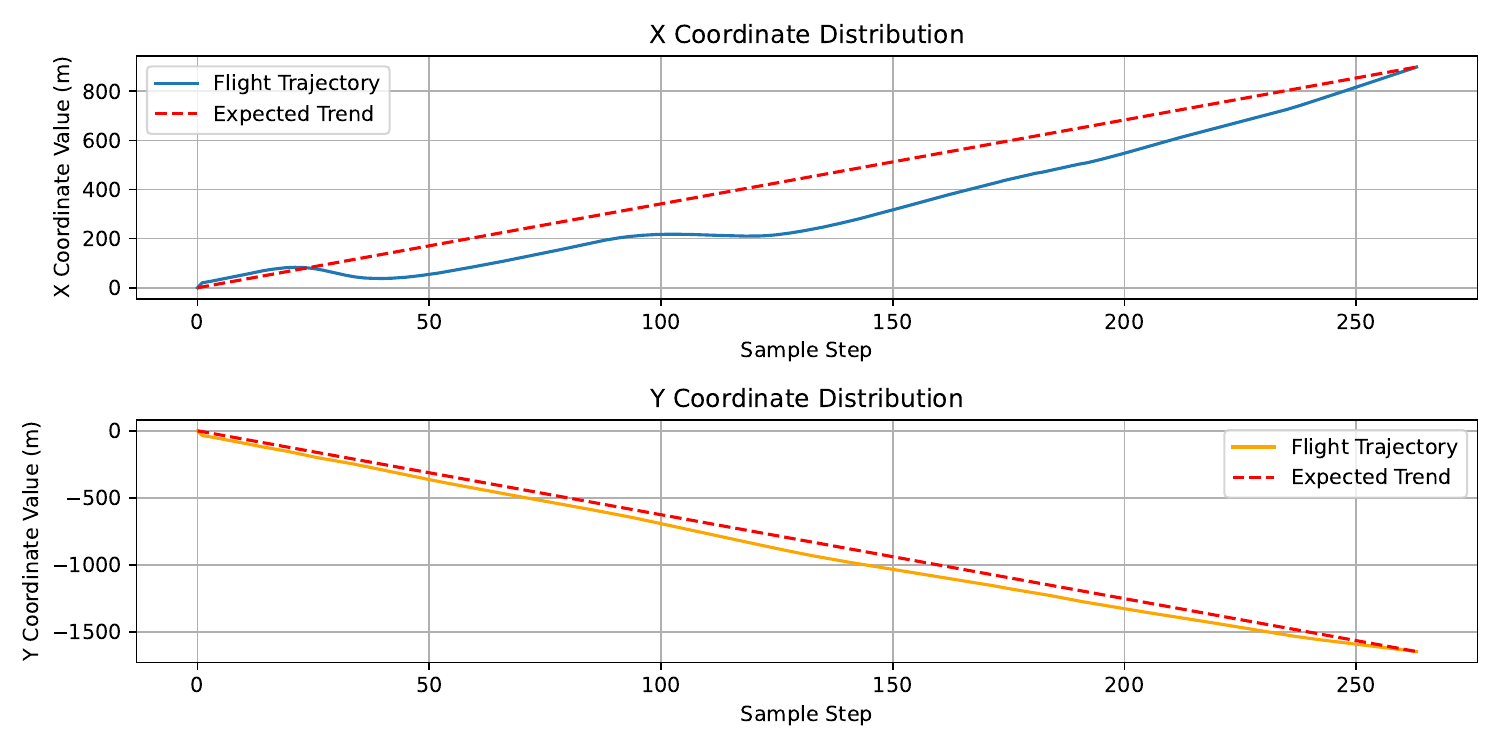} 
        \caption*{(b) Scene II}
        \label{fig:sub2}
    \end{minipage}
    \hfill
    \begin{minipage}[b]{0.48\textwidth} \centering
        \includegraphics[width=\linewidth]{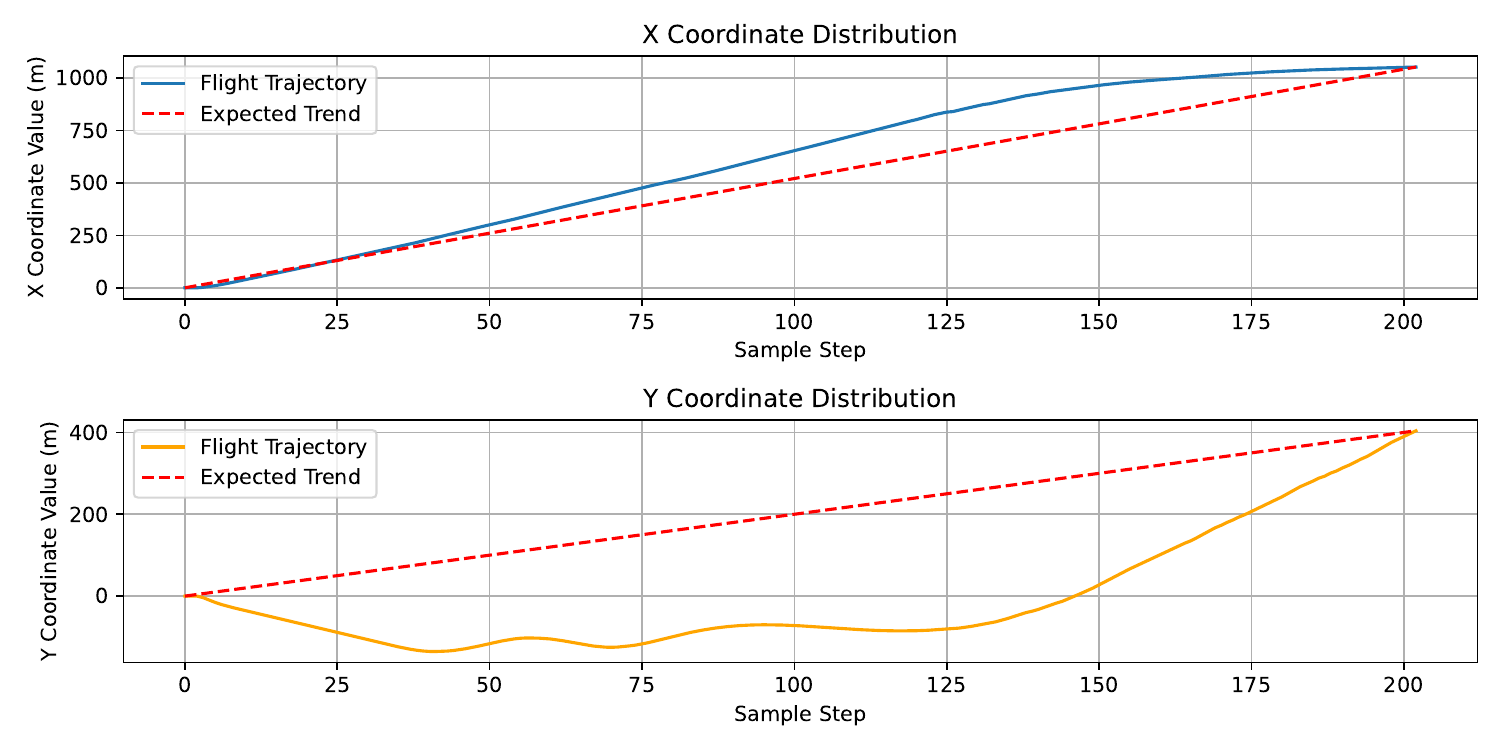} 
        \caption*{(c) Scene III}
        \label{fig:sub3}
    \end{minipage}

    \caption{HIL flight trajectory and coordinate distributions across three distinct scenes. Each scene shows the fixed-wing UAV's trajectory in a virtual environment along with its X and Y coordinate distributions compared to the expected trend.}
    \label{fig:qu_Analists}
\end{figure}
and dynamic constraints of fixed-wing UAVs, along with an adaptive entropy-based strategy update mechanism to balance exploration and exploitation during training. 
The proposed method is tested in various scenarios through hardware-in-the-loop simulations and compared with other reinforcement learning algorithms. 
The experimental results demonstrated that our framework significantly outperforms these algorithms in terms of obstacle avoidance effectiveness and trajectory smoothness. 
Despite the promising results, our study has certain limitations. The reliance on an inferred depth map may affect the accuracy of obstacle detection, particularly in environments with sudden, small obstacles. In the future, we plan to deploy the proposed algorithm on a real vertical take-off and landing (VTOL) fixed-wing UAV to validate its feasibility in real-world scenarios.


\appendices


\ifCLASSOPTIONcaptionsoff
  \newpage
\fi



%
\bibliographystyle{unsrt}
\bibliography{references}

%





\end{document}